\documentclass{article} % For LaTeX2e
\usepackage{iclr2024_conference,times}
\usepackage{enumitem}

\usepackage{times}
\usepackage{latexsym}
 \usepackage{bbm}
\usepackage[T1]{fontenc}
\iclrfinalcopy

\usepackage[utf8]{inputenc}
\usepackage{url}
\usepackage{hyperref}

\usepackage{microtype}
\usepackage{graphicx}
\usepackage{comment}
\usepackage{blkarray}

\input{Definitions}
\title{CausalLM is not optimal for in-context learning}
\author{Nan Ding \quad Tomer Levinboim \quad Jialin Wu \quad
Sebastian Goodman \quad Radu Soricut \\
  Google Research \\
  \texttt{\{dingnan,tomerl,jialinwu,seabass,rsoricut\}@google.com}}

\begin{document}
\maketitle
\begin{abstract}
Recent empirical evidence indicates that transformer based in-context learning performs better when using a prefix language model (prefixLM), in which in-context samples can all attend to each other, compared to causal language models (causalLM), which use auto-regressive attention that prohibits in-context samples to attend to future samples. While this result is intuitive, it is not understood from a theoretical perspective. In this paper we take a theoretical approach and analyze the convergence behavior of prefixLM and causalLM under a certain parameter construction. Our analysis shows that both LM types converge to their stationary points at a linear rate, but that while prefixLM converges to the optimal solution of linear regression, causalLM convergence dynamics follows that of an online gradient descent algorithm, which is not guaranteed to be optimal even as the number of samples grows infinitely. We supplement our theoretical claims with empirical experiments over synthetic and real tasks and using various types of transformers. Our experiments verify that causalLM consistently underperforms prefixLM in all settings.

\end{abstract}

\section{Introduction}
Transformer-based models~\citep{NIPS2017_3f5ee243} have become the default foundational model for various machine learning applications such as natural language processing~\citep{devlin2018bert,brown2020language,chowdhery2022palm} and computer vision~\citep{dosovitskiy2020image}. Beyond their traditional usage in machine learning applications, it has recently been discovered that pretraining large transformers on a vast amounts of data leads them to develop a striking ability 
referred to as in-context learning (ICL)~\citep{brown2020language}. Specifically, once such pretraining is complete, these models are able to solve new tasks at inference time (without changing their parameters) by simply ingesting a short sequence (prefix) of labeled examples from a task and then computing a prediction for a query example.

The ICL capability was first demonstrated by GPT-3~\citep{brown2020language}, where a causalLM (a Transformer decoder with auto-regressive attention masks) was used as the main model architecture. 
However, follow up work empirically found that restricting the auto-regressive masks on the entire sequence is too prohibitive and therefore proposed the so-called prefixLM~\citep{raffel2020exploring,tay2022unifying} which allows full attention within the prefix tokens.
Moreover, the latest models (such as PaLM2~\citep{palm2}) adopt a mixture of different LM objectives during pretraining to achieved state-of-art performance across a diverse set of tasks and capabilities.

However, beyond the few empirical results in those and related papers, there is yet no theoretical explanation that accounts for the different ICL behavior of prefixLM and causalLM.
Indeed, theoretical studies of ICL are difficult due to the complicated non-linearity of the (ordinary) transformer architecture. However, recent work~\citep{von2023transformers} focusing on ICL of linear regression was able to show
that a specifically designed parameter construction of a one-layer Linear Self-Attention (LSA) transformer can simulate a single step of gradient descent by using the in-context examples as training data.
Moreover, a different recent study~\citep{zhang2023trained} used gradient flow 
to prove that a randomly initialized LSA-transformer indeed converges to such a construction during training. 

In this paper, we continue the theoretical line of work above by investigating the convergence properties of ICL for both prefixLM and causalLM multi-layer LSA-transformers in a linear regression setting.
We summarizes our contributions as follows:
\begin{itemize}
\item We first present a clear, formal proof that establishes the relationship between a multi-layer LSA and multi-step gradient descents in linear regression. 
\item We then show that both causalLM and prefixLM based multi-layer LSA-transformers converge to their respective stationary points with linear rates of convergence.
% with increasing layers. 
We prove that the stationary point of prefixLM corresponds to the optimal least square solution of the linear regression problem, while the stationary points of causalLM correspond to the weights obtained along the iterations of online gradient descent with non-decaying step sizes. 
Importantly, the stationary points obtained by causalLM may not become optimal even as the number of in-context examples increases, which indicates that causalLM is not optimal for in-context learning. 
\item Finally, we verify the above theoretical insights by conducting experiments with LSA-transformers as well as ordinary softmax attention based transformers on various synthetic tasks including linear and non-linear regression, and multiclass classifications. We also compare causalLM and prefixLM ICL based on LLMs including T5~\citep{roberts2022t5x} and PaLM2~\citep{palm2}, as well as the multimodal model PaLI-X~\citep{Chen2023PaLIXOS}.
%by first finetuning it with either prefixLM or causalLM over FLAN tasks~\citep{chung2022scaling} and then evaluating on MMLU~\citep{hendrycks2020measuring} 5-shot tasks. 
Our experimental results support our theoretical findings and consistently show the superiority of prefixLM over causalLM on ICL for such settings.
\end{itemize}

\section{Background}
\label{sec:background}
We begin by reviewing a few types of transformer attention and in-context learning (ICL), as well as a specific transformer construction for linear regression ICL by~\citep{von2023transformers} which our theories will be based on. 
The discussions of other related work are deferred to Appendix~\ref{sec:related}.

\subsection{Transformers: SSA, LSA, causalLM, and prefixLM}
\label{sec:transformers}
Given a sequence of input vectors $\Zb = \rbr{\zb_1, \ldots, \zb_n}$, the output of standard Softmax Self-Attention (SSA) layer is
\begin{align*}
    \zb_j \leftarrow \zb_j + \Pb \Vb \Zb \text{softmax}(\Zb^{\top} \Kb^{\top} \Qb \zb_j),
\end{align*}
where $\Pb, \Vb, \Kb, \Qb$ respectively corresponds to the output projection, value transformation, key transformation and query transformation. 

Since the softmax attention of standard transformers is non-linear, its theoretical analysis becomes complicated even for a single layer. 
For this reason, theoretical approaches to analyze transformers have often resorted to the Linear Self-Attention (LSA) layer~\citep{von2023transformers,zhang2023trained}, which simply drops the softmax function from the attention, 
\begin{align}
    \zb_j \leftarrow& \zb_j + \Pb \Vb \Zb (\Zb^{\top} \Kb^{\top} \Qb \zb_j)
    = \zb_j + \Pb \Vb \sum_{i=1}^n \zb_i \rbr{\zb_i^{\top} \Kb^{\top} \Qb \zb_j}. \label{eq:lsa}
\end{align}
Furthermore, since each input $\zb_j$ can attend to all positions $j \in \{1 \ldots n\}$, this form of attention is categorized as full (or bidirectional) attention, and is typically used in the transformer encoder.

On the other hand, a (linear) transformer decoder uses the \emph{auto-regressive} attention 
\begin{align}
    \zb_j \leftarrow \zb_j + \Pb \Vb \sum_{i=1}^j \zb_i \rbr{\zb_i^{\top} \Kb^{\top} \Qb \zb_j}. \label{eq:causal_lsa}
\end{align}
which restricts each token $\zb_j$ to attend only to previous positions (and itself) from $\{1 \ldots j\}$.
This restriction is due to the role of the decoder as a causal language model (causalLM) which predicts the next token in the context of the previously generated ones.

The original transformer involves both a full attention based encoder and an auto-regressive attention based decoder. However, prominent NLP research has often chosen either encoder-only (e.g. BERT~\citep{devlin2018bert}) or decoder-only (e.g. GPT~\citep{brown2020language}, PaLM~\citep{chowdhery2022palm}) models according to the task at hand. This is partially for the purpose of halving the parameter sizes. 

Another version of  attention, between full and auto-regressive, followed from the observation that some tasks can benefit from a prefix sequence such as context or prompt. That is, the input sequence $\Zb$ is composed of $n'$ prefix tokens $(\zb_1, \ldots, \zb_{n'})$ configured for the task, while the tokens $(\zb_{n'+1}, \ldots, \zb_n)$ represent the sample.
Specifically, prefixLM~\citep{raffel2020exploring} suggests the following attention (in its LSA version):
\begin{align*}
    \zb_j \leftarrow \zb_j + \Pb \Vb \sum_{i=1}^{\max(j, n')} \zb_i \rbr{\zb_i^{\top} \Kb^{\top} \Qb \zb_j},
\end{align*}
where $\max(j, n')$ ensures each prefix token $\zb_j$ with $j < n'$ can attend to all prefix tokens. 
%from $(\zb_1, \ldots, \zb_{n'})$.
% Therefore, it seems unnecessary to limit the attention of such prefixed tokens to its previous tokens only, so the prefixLM has been developed~\citep{raffel2020exploring}. 
% Roughly speaking, prefixLM can be viewed as an encoder-decoder model with tied parameters. 

\subsection{In-context learning}
\label{sec:icl}
A formal framework of in-context learning has been described in various existing literature such as \citep{garg2022can,zhang2023trained}. Here, we briefly review the problem setting and introduce  notation that will be used across the paper.

In-context learning refers to the 
ability of models to produce context-driven predictions at inference time.
%predicting tasks of a fixed-parameter model, given a sequence of training data inputs and labels in the context. 
That is, at inference time, a model is fed with a sequence consisting of input-label pairs and a query input $(\xb_1, y_1, \ldots, \xb_n, y_n, \xb_{query})$ and its goal is to predict the label $y_{query}$ of $\xb_{query}$ using the context examples $(\xb_1, y_1, \ldots, \xb_n, y_n)$ (specifically, without changing the model parameters).

\subsection{Linear regression in-context learners}
\label{sec:lricl}
Linear regression is a classical machine learning problem. Given a set of input-label pairs $\rbr{\xb_i, y_i}$, the goal is to find an optimal weight vector $\wb$ that minimizes the l2-loss: 
$$L(\wb) = \frac{1}{2n}\sum_{i=1}^n \| \wb\xb_i - y_i \|_2^2.$$ 
The gradient of the loss is $\nabla_{\wb} L = \frac{1}{n}\sum_{i=1}^n (\wb\xb_i - y_i)\xb_i^{\top}$, and a gradient descent algorithm with step size $\eta$ follows the update rule:
\begin{align}
\wb^{(l)} %=& \; \wb^{(l-1)} - \eta \nabla_{\wb} L(\wb^{(l-1)}) \nonumber\\ 
=& \wb^{(l-1)} + \frac{\eta}{n} \sum_{i=1}^n (y_i - \wb^{(l-1)} \xb_i) \xb_i^\top. \label{eq:prefix_update_w}
\end{align}
Using linear regression as a lens to study in-context learning was first proposed in ~\citep{garg2022can}, where the authors laid out an approach for training transformers to in-context learn a class of simple predictors, including linear regression. However, no theoretical study was provided.
More recently, and most relevant to our work, ~\citep{von2023transformers} proposed a succinct construction that demonstrates how a single LSA layer can effectively implement a single gradient descent step. According to their setup the input is formulated as
\begin{align}
%    \Zb = 
%    \begin{pmatrix}
%\xb_1 & \cdots & \xb_n & \xb_{query}\\
% y_1 & \cdots & y_n & 0
%\end{pmatrix},  
\Zb = (\zb_1^{(0)}, \ldots, \zb_n^{(0)}), \;
\text{where}\;
\zb_j^{(0)} = \begin{pmatrix}\xb_j \\ y_j \end{pmatrix} %\; \zb_{query}^{(0)} = \begin{pmatrix}\xb_{query} \\ 0 \end{pmatrix}
\label{eq:icl_input}
\end{align}
and the parameter matrices of \eqref{eq:lsa} are set as:
\begin{align}    
\Kb &= \Qb = \begin{pmatrix} \Ib_{d \times d} & {\bf 0} \\ 0 & 0\end{pmatrix}, \Vb = \begin{pmatrix} {\bf 0}_{d \times d} & {\bf 0} \\ \wb^{(0)} & -1\end{pmatrix}, \Pb = \frac{\eta}{n} \Ib, \label{eq:constructed_lsa}
\end{align}
where $\wb^{(0)}$ is an initial weight vector.
\citep{von2023transformers} then showed that this configuration results in an update of their so-called transformed target $y_j \leftarrow y_j + \eta \rbr{\nabla_{\wb^{(0)}} L} \xb_j$, and that this target update is equivalent to the one performed by a single-step gradient descent of linear regression.

Although the construction of~\citep{von2023transformers} connected LSA-based ICL to the gradient descent of linear regression, the "transformed target" view seems unnatural\footnote{The traditional ML formulation updates the weight vector or the model prediction, while the groundtruth target remains fixed.} to work with. Moreover, their extension from single-layer to multi-layer LSA is unfortunately unclear. 
%For example, it seems to suggest that $\Vb$ of one LSA layer depends on the $\wb$ of the previous layer, but the latter becomes data-dependent after the first gradient update.
% Although $\wb^{(0)}$ can be preset to a constant (e.g. {\bf 0}) in the initial step, its value becomes data-dependent after the first gradient update.

\section{Multi-layer in-context learner}
\label{sec:ml-icl}

In this section, we provide a formal proof that a multi-layer LSA under the construction of~\citep{von2023transformers} progresses identically to  multi-step gradient descent. 

Instead of the "transformed target" view, the following proposition explicitly connects the GD weights of~\eqref{eq:prefix_update_w} to the outputs of the multi-layer LSA under the constructions of $\Kb$, $\Qb$, $\Pb$ and $\Vb$ in \eqref{eq:constructed_lsa}. 
Note that we keep $\wb^{(0)}=0$ in the proposition because it simplifies the equations and makes the outputs more meaningful. However, such specification is not mandatory, and we provide general propositions, for arbitrary $\wb^{(0)}$, in Appendix~\ref{sec:general_w0}.
\begin{proposition}
\label{prop:muicl}
For a multi-layer LSA satisfying the construction \eqref{eq:constructed_lsa} and with $\wb^{(0)} = 0$, if its input $\Zb$ is formatted as~\eqref{eq:icl_input}, then its $l$-th layer output is $\zb_j^{(l)} = (\xb_j^\top, \delta_j^{(l)})^\top$, where $\delta_j^{(l)} = y_j - \wb^{(l)} \xb_j$ and $\wb^{(l)}$ is the $l$-th updated weight from the gradient descents update rule in~\eqref{eq:prefix_update_w}.
\end{proposition}
\emph{Proof Sketch: }
Plugging in $\Kb$, $\Qb$, $\Pb$ and $\Vb$ of \eqref{eq:constructed_lsa} with $\wb^{(0)}=0$ and $\zb_j^{(l)} = (\xb_j^\top, \delta_j^{(l)})^\top$ into \eqref{eq:lsa}, we obtain that for all $l>0$,
\begin{align*}
   &\begin{pmatrix}\xb_j \\ \delta_j^{(l)} \end{pmatrix} = \begin{pmatrix}\xb_j \\ \delta_j^{(l-1)} \end{pmatrix} - \frac{\eta}{n} \sum_{i=1}^n \begin{pmatrix}{\bf 0} \\ \delta_i^{(l-1)} \end{pmatrix} \xb_i^\top \xb_j.
\end{align*}
Since $\zb_j$ never changes its first $d$-dimension corresponding to $\xb_j$, we can simplify it and focus only on $\delta_j^{(l)}$, which is the last output coordinate of the $j$-th LSA-layer,
\begin{align}
\delta_j^{(l)} =& \; \delta_j^{(l-1)} - \frac{\eta}{n} \sum_{i=1}^n \delta_i^{(l-1)} \xb_i^\top \xb_j, \label{eq:prefix_update_delta}
\end{align}
with $\delta_j^{(0)} = y_j$. Defining $\tilde{y}_j^{(l)} = y_j - \delta_j^{(l)}$ and rearranging \eqref{eq:prefix_update_delta}, we obtain $\tilde{y}_j^{(0)} =0$ and $\forall l>0$:
\begin{align}
\tilde{y}_j^{(l)} =& \; \tilde{y}_j^{(l-1)} + \frac{\eta}{n} \sum_{i=1}^n (y_i - \tilde{y}_i^{(l-1)}) \xb_i^\top \xb_j. \label{eq:prefix_update_y}
\end{align}
Finally, using \eqref{eq:prefix_update_y} and the fact that $\tilde{y}_j^{(0)} = 0 = \wb^{(0)} \xb_j$, it can be proved by induction that $\forall l: \tilde{y}_j^{(l)} = \wb^{(l)} \xb_j$. A complete proof is provided in Appendix \ref{sec:proof}. 

To summarize, the newly introduced variable $\tilde{y}_j^{(l)}$ is exactly the prediction of the $l$-th gradient descent weights $\wb^{(l)}$ for $\xb_j$ , and $\delta_j^{(l)}$ is the difference between the true label $y_j$ and the predicted $\tilde{y}_j^{(l)}$. Therefore, $\tilde{y}_j^{(l)}$ serves as a bridge to connect the LSA output $\delta_j^{(l)}$ and the GD weight $\wb^{(l)}$. 

So far, we have dealt with the behavior of LSA layers with full attention. In what follows, we move on to the practical setting of in-context learning, where the input contains not only $n$ in-context (training) examples in the format of \eqref{eq:icl_input}, but also an additional (test) query $\zb_{query}^{(0)} = (\xb_{query}^\top, 0)^\top$. In particular, we will focus on the two common ICL variants: prefixLM and causalLM, each with a different type of attention.

\subsection{PrefixLM ICL}

A prefixLM ICL treats the in-context examples $\Zb$ as the prefix and uses full attention on the first $n$ positions, so that they can each freely attend to each other. 
%If the input is in the format of \eqref{eq:icl_input}, then we keep the $n$ input vectors as in-context examples, and append an additional $\zb_{query}^{(0)} = (\xb_{query}^\top, 0)^\top$ to represent the query (see Figure~\ref{fig:ml-icl}). The in-context examples attend to each other. 
The last query vector $\zb_{query}$ can also attend to any example in $\Zb$, but cannot attend to itself\footnote{This is because the query does not contain a meaningful label. Attending to itself would cause it to include its last-dim input as a label, which would contaminate the resulting multi-layer prediction. This observation was not considered in~\citep{von2023transformers}.}. 
As a result, the updates of the prefixLM-ICL under the same construction follow \eqref{eq:prefix_update_delta}, with the outputs of the $l$-th layer being,
\begin{align*}
    \delta_j^{(l)} &= y_j - \tilde{y}_j^{(l)} = y_j - \wb^{(l)} \xb_j,\\
   \text{and}\;\; \delta_{query}^{(l)} &= -\tilde{y}_{query}^{(l)} = - \wb^{(l)} \xb_{query},
\end{align*}
where the initial $\tilde{y}_j^{(0)} = \tilde{y}_{query}^{(0)} = 0$. 

Intuitively, the dynamics of the prefixLM ICL is as follows: all $\tilde{y}_j^{(l)}$ starts as 0 at $l=0$, 
and gradually approach to the true label $y_j$ as $l$ increases, so that the difference (also as the output) $\delta_j^{(l)}$ gradually approaches to 0. At the same time, $\delta_{query}^{(l)}$ starts at 0, and gradually approaches to $-y_{query}$, the negation of the query label. Figure~\ref{fig:ml-icl} provides an illustration of these dynamics.

\begin{figure}
\vspace{-0.3in}
  \centering
  \includegraphics[width=0.47\linewidth]{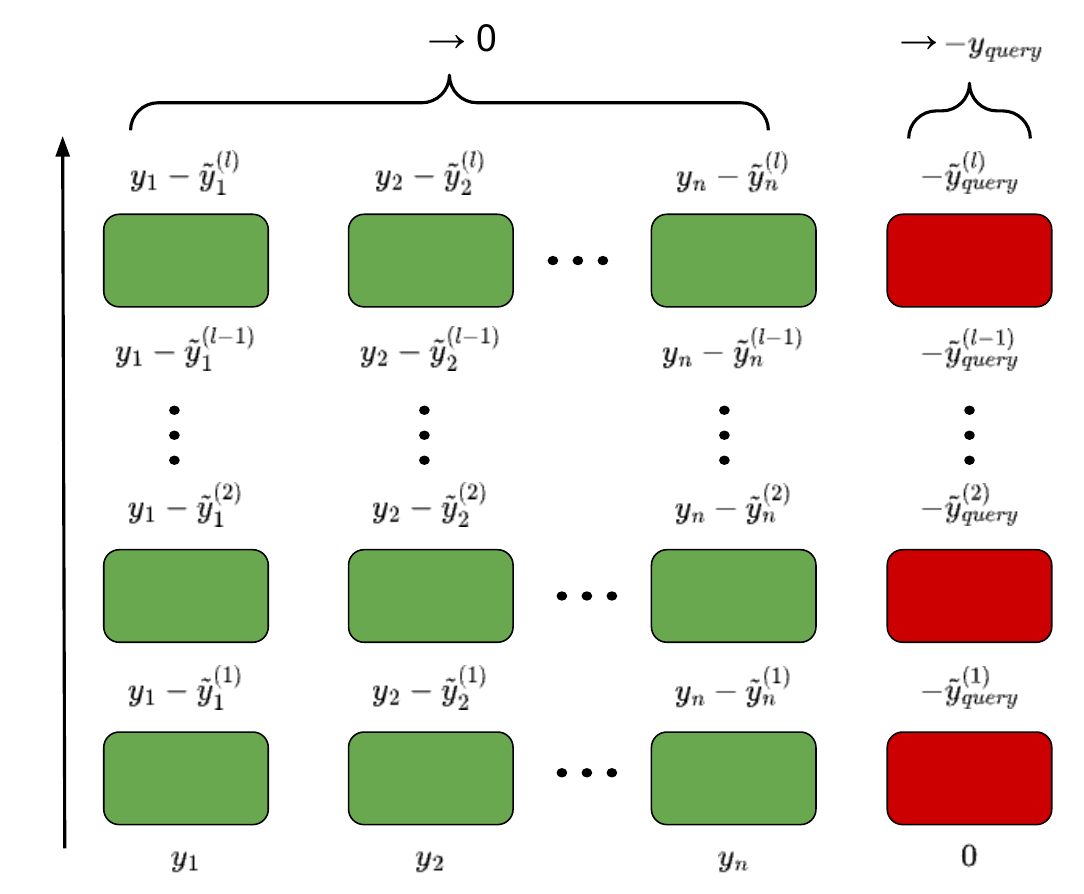}
   \caption{The inputs/outputs of a multi-layer in-context learner. We omitted $\xb_j$ and $\xb_{query}$ since they are unchanged.}
   \label{fig:ml-icl}
\end{figure}

\subsection{CausalLM ICL}
\label{sec:causallm}
A causalLM applies auto-regressive attention throughout the entire sequence. Therefore, plugging the same $\Kb$, $\Qb$, $\Pb$, $\Vb$ into \eqref{eq:causal_lsa}, the update rules of \eqref{eq:prefix_update_delta} and \eqref{eq:prefix_update_y} become:
\begin{align}
    \delta_j^{(l)} =& \; \delta_j^{(l-1)} - \frac{\eta}{n} \sum_{i=1}^j \delta_i^{(l-1)} \xb_i^\top \xb_j, \label{eq:causal_update_delta}\\ 
\tilde{y}_j^{(l)} =& \; \tilde{y}_j^{(l-1)} + \frac{\eta}{n} \sum_{i=1}^j (y_i - \tilde{y}_i^{(l-1)}) \xb_i^\top \xb_j \label{eq:causal_update_y}
\end{align}
\footnote{There is another way of update which changes $\eta/n$ to $\eta/j$ for the $j$-th example. We provide more details in Appendix~\ref{sec:causallm2} and show it performs worse than the main version in \eqref{eq:causal_update_delta}.}with $\delta_j^{(l)} = y_j - \tilde{y}_j^{(l)}$. 
Moreover, since different $\delta_j, \tilde{y}_j$ are exposed to different ranges of inputs, there is no uniform $\wb$ as in~\eqref{eq:prefix_update_w} that is associated with all $\tilde{y}_j$. Instead, if we define $\wb_j$ for each different position $j$ with $\wb_j^{(0)}=0$ and
\begin{align}    
\wb_j^{(l)} = \wb_j^{(l-1)} + \frac{\eta}{n} \sum_{i=1}^j (y_i - \wb_i^{(l-1)} \xb_i) \xb_i^\top \label{eq:causal_update_w}
\end{align}
then we have the following proposition: 
\begin{proposition}
\label{prop:muicl-causal}
For a multi-layer causalLM-LSA satisfying \eqref{eq:constructed_lsa} with $\wb^{(0)} = 0$, if its input $\Zb$ is formatted as~\eqref{eq:icl_input}, then its $l$-th layer output is $\zb_j^{(l)} = (\xb_j^\top, \delta_j^{(l)})^\top$, where $\delta_j^{(l)} = y_j - \wb_j^{(l)} \xb_j$ and $\wb_j^{(l)}$ follow~\eqref{eq:causal_update_w}.
\end{proposition}

%then it is easy to verify $\tilde{y}_j^{(l)} = \wb_j^{(l)} \xb_j$, as well as
The proof of Proposition~\ref{prop:muicl-causal} is provided in Appendix~\ref{sec:proof}. Similar to prefixLM-ICL, causalLM-ICL also has $\tilde{y}_j^{(l)} = \wb_j^{(l)} \xb_j$, and
\begin{align*}
%\delta_j^{(l)} &= y_j - \tilde{y}_j^{(l)} = y_j - \wb_j^{(l)} \xb_j, \\
%\text{and}\;\; 
\delta_{query}^{(l)} &= - \tilde{y}^{(l)}_{query} = -\wb_n^{(l)} \xb_{query}.
\end{align*}
%Since prefixLM-ICL and causalLM-ICL are associated with different update rules: $\wb^{(l)}$ follows~\eqref{eq:prefix_update_w} and $\wb_j^{(l)}$ follows~\eqref{eq:causal_update_w}, their outcomes are significantly different. In the next section, we look closely into their convergence properties, and show why prefixLM is superior to causalLM for in-context learning.

In summary, causalLM-ICL and prefixLM-ICL are associated with different update rules: $\wb_j^{(l)}$ follows~\eqref{eq:causal_update_w} while $\wb^{(l)}$ follows~\eqref{eq:prefix_update_w}. Specifically, in causalLM, it can be seen that the $\wb_i^{(l-1)}$ corresponding to the first positions are biased due to restricted access to only a few data points and furthermore, that these biases are propagated to later positions by \eqref{eq:causal_update_w}. 
In prefixLM on the other hand, each position has access to all the data and a single $\wb^{(l)}$ can be used across the entire sequence as in \eqref{eq:prefix_update_w}. 
Although Eq. \eqref{eq:prefix_update_w} and Eq. \eqref{eq:causal_update_w} only hold for the structured LSA case, the profound difference between causalLM and prefixLM stems from their architectural difference and therefore we believe extends to general transformers, as indicated by our experimental results in Section~\ref{sec:experiment}.

\section{Convergence of the multi-layer in-context learners}
\label{sec:theory-ml-icl}
In this section, we prove that both multi-layer prefixLM and causalLM converge to their respective stationary points with increasing layers (and with linear rates). In addition, we show that the stationary point of prefixLM corresponds to the 
optimal % <===
least-square solution of the linear regression problem, while the ones corresponding to causalLM are equivalent to the iterative weights of online gradient descent of linear regression, which are 
known to be sub-optimal for a limited number of examples.
%inferior to prefixLM.

\subsection{Convergence of the prefixLM ICL}
\label{sec:theory-prefixlm}
The fact that a multi-layer prefixLM computation exactly follows the update rule of $\wb^{(l)}$ as in \eqref{eq:prefix_update_w}, implies that the layer outputs of prefixLM have the same dynamics of 
multi-step gradient descent on a linear regression problem.
% Since the multi-layer prefixLM implicitly updates the weight vector $\wb^{(l)}$ with \eqref{eq:prefix_update_w}, 
% it exactly follows the dynamics of the multi-step gradient descent on a linear regression problem. 
The convergence properties of such dynamics are well-known, and are  stated in the following proposition:
\begin{proposition}
\label{prop:prefix-icl}
If $\wb^{(l)}$ follows the iterative updates of \eqref{eq:prefix_update_w}, then there exists a stationary point $\wb^*$ with coefficients satisfying:
\begin{align*}
   \yb \Xb^\top = \wb^* \Xb \Xb^\top,
\end{align*}
where $\yb = (y_1, \ldots, y_n)$ and $\Xb = (\xb_1, \ldots, \xb_n)$. Furthermore, the iterative weights $\wb^{(l)}$ converge to $\wb^*$ with a linear rate of convergence:
\begin{align*}
   \wb^{(l)} - \wb^* = (\wb^{(l-1)} - \wb^*)(\Ib - \frac{\eta}{n} \Xb \Xb^\top ).
\end{align*}
\end{proposition}
That is, Proposition \ref{prop:prefix-icl} holds for the multi-layer prefixLM, so that the same exact $\wb^*$ is also the stationary point of prefixLM, to which it converges in a linear rate.
Furthermore this stationary point is exactly the (optimal) least square solution of the linear regression problem.

\subsection{Convergence of the causalLM ICL}
\label{sec:theory-causallm}
Following the update rule of \eqref{eq:causal_update_w}, we can view a multi-layer causalLM as implicitly maintaining different weight vectors $\wb_j$ for each position $j$.
In what follows, we show that: (a) Each such position $j$ has its own stationary point $\wb_j^*$, which appears to be different from the global optimal point $\wb^*$ of linear regression; (b) even when the number of in-context samples $n$ grows to infinity, convergence to $\wb^*$ is not guaranteed.

% Due to its auto-regressive attention, we can view a multi-layer causalLM as implicitly maintaining different weight vectors $\wb_j$ for each position $j$, each following  
% the update rule \eqref{eq:causal_update_w} at its position $j$. 
%A careful investigation of this update rule reveals that each position $j$ has its own stationary point to which it converges with a linear rate. 
Specifically, in Appendix~\ref{sec:proof} we provide a proof for the following proposition: 
% As a result, we obtain the following proposition regarding the stationary points and converging dynamics of $\wb_j$:
\begin{proposition}
\label{prop:causal-icl}
If $\wb_j^{(l)} = \sum_{i=1}^j a_{i,j}^{(l)} \xb_i^\top$ follows the iterative updates of \eqref{eq:causal_update_w}, then 
\begin{align*}
a_{i,j}^{(l)} = a_{i,i}^{(l)} \equiv a_i^{(l)} \;\;\;\; \forall j \ge i,
\end{align*}
and there exist stationary points $\wb_j^* = \sum_{i=1}^j a_i^* \xb_i^\top$ (for $j \in 1, \ldots, n$) with coefficients from $\ab^* = (a^*_1, \ldots, a^*_n)$ that satisfy $\yb = \ab^* \Tb$, where
\begin{align*} 
\Tb = \begin{pmatrix}
\xb_1^\top \xb_1 & \xb_1^\top \xb_2 & \cdots & \xb_1^\top \xb_n \\
0 & \xb_2^\top \xb_2 & \cdots & \xb_2^\top \xb_n \\
\vdots & \vdots & \ddots & \vdots \\
0 & 0 & \cdots & \xb_n^\top \xb_n 
\end{pmatrix}.
\end{align*}
Furthermore, the coefficients $\ab^{(l)}$ converges to the stationary point $\ab^*$ with linear rate of convergence:
\begin{align*}
   \ab^{(l)} - \ab^* = (\ab^{(l-1)} - \ab^*)(\Ib - \frac{\eta}{n} \Tb ).
\end{align*}
\end{proposition}
%The proof of Proposition~\ref{prop:causal-icl} is provided in Appendix~\ref{sec:proof}. 
%This proposition implies that multi-layer causalLM ICL also has a stationary point and converges to it in linear rates with respect to the number of layers. 
This proposition implies that the stationary points $\wb_j^*$ of causalLM-ICL are 
different from $\wb^*$, the least square solution of linear regression. However, a natural question is: if $j$ increases, would $\wb_j^*$ ultimately converge to the optimal solution?

To answer this question, the next proposition shows that the stationary points $\wb_j^*$ follow an online gradient descent algorithm, whose loss and gradient at the $j$-th step is,
\begin{align*}
    L_j(\wb_j) &= \frac{1}{2} (\wb_j \xb_{j+1} - y_{j+1})^2, \\
    \nabla_{\wb_j} L_j(\wb_j) &= (\wb_j \xb_{j+1} - y_{j+1}) \xb_{j+1}^\top.
\end{align*}
\begin{proposition}
\label{prop:causal-ogd}
Assuming that $\wb_j^*$ is the stationary points obtained in Proposition \ref{prop:causal-icl}, then 
\begin{align}
    \wb_{j+1}^* = \wb_j^* - \frac{1}{\|\xb_{j+1}\|_2^2} \nabla_{\wb_j^*} L_j(\wb_j^*).  \label{eq:causal-ogd}
\end{align}
\end{proposition}
The proof of Proposition~\ref{prop:causal-ogd} is provided in Appendix~\ref{sec:proof}. Note that online gradient descent is known to converge to an optimal solution only with a decaying step size $j^{-\nu}$ for $\nu > 0$~\citep{jentzen2020lower}.
Since the step size of \eqref{eq:causal-ogd} does not decay, we conclude that causalLM may not converge to 
$\wb^*$ even with increasing layers and increasing number of in-context examples. More concretely, as for the case of in-context learning, where the number of in-context examples $n$ is limited, convergence to the optimal solution $\wb^*$ cannot be achieved by causalLM-ICL.

%In conclusion, we have shown that under a certain parameter construction of the transformer, each layer of the transformer is equivalent to a single gradient descent update rule (for linear regression).
%While in the prefixLM case, these steps lead to the optimal solution, in the causalLM case these steps remain biased due to restricted access of the first positions to few data points.
% but, did the reviewer need our proof to realize this?
%Furthermore, although our propositions hold only for the structured LSA case, we believe that the profound difference they reveal between prefixLM and causalLM holds also for the standard softmax-attention transformer.

\section{Numerical experiments}
\label{sec:experiment}
Our experiments contain three parts. 
\begin{itemize}[leftmargin=15pt]
    \item We first use LSA-transformers on linear regression problems to directly verify our theorems. In Section~\ref{sec:lsa-lr}, we show that despite that the in-context example (training) error of causalLM and prefixLM both decays in linear rates, the query (test) error of causalLM is significantly larger, which indicates its stationary solution is not optimal. 
    \item Secondly, we use ordinary softmax transformers on a few synthetic tasks including linear regression, nonlinear regression and multiclass classification. In Section~\ref{sec:transformer-synthetic}, we show that our theoretical insights generalize to other tasks types (i.e., that ICL prefixLM still outperforms causalLM in all these cases).
    \item Lastly, in Section~\ref{sec:flan-t5x}, we conduct LLM based ICL experiments using T5 \citep{roberts2022t5x}. We also provide additional experimental results on PaLM2~\citep{palm2} as well as large multimodal models (PaLI-X,~\cite{Chen2023PaLIXOS}) in Appendix~\ref{sec:palm2} and~\ref{sec:palix}.
\end{itemize}

\subsection{LSA-transformers on linear regression}
\label{sec:lsa-lr}
In order to directly verify our theorems from Section \ref{sec:theory-ml-icl}, we first study in-context learning on linear regression problem with the LSA transformer of~\eqref{eq:constructed_lsa}. 
Each of the input sequence contains 40 in-context examples and 200 queries, and each query attends to all the in-context examples but does not attend to each other. See Appendix \ref{sec:exp_details} for an illustration. The data input $\xb_i$ of the sequence is sampled from $\Ucal(-1, 1)^{16}$. Each sequence is associated with a single weight vector $\wb$ that is sampled from $\Ncal(0, \Ib)$, and the labels are computed as $y_i = \wb \xb_i$. 
Assuming the prediction of each layer is $\tilde{y}_i^{(l)}$, we evaluate the MSE $\|y_i - \tilde{y}_i^{(l)}\|_2^2$ on both in-context and query examples across different layers $l$. 

The results are plotted in Figure~\ref{fig:icl-convergence} left (for prefixLM) and middle (for causalLM). Our results are averaged over 64 randomly generated sequences. As we can see, although both prefixLM and causalLM has a linear rate of convergence (with respect to the number of layers) on the in-context examples, the query errors of causalLM are stuck  above the $10^{-1}$ level, while the query error of prefixLM decays in the same linear rate as its training error.

Furthermore, in Figure~\ref{fig:icl-convergence} right, we plot the query errors of the stationary points (following Proposition~\ref{prop:causal-icl}, corresponding to the outputs of infinite layers) of causalLM-ICL with increasing number of in-context examples up to 300. Although causalLM-ICL is able to eventually converge to optimal solution when $\mu_x=0$, it takes more than 100 examples to reach below $10^{-2}$. The convergence is even worse as we vary the input distribution $\xb \sim \Ucal(-1, 1)^d + \mu_x$ with increasing $\mu_x \in \cbr{0, 1, 2, 3}$, which demonstrates that causalLM-ICL is not optimal for few-shot ICL.

\begin{figure*}[t]
  \centering
  \includegraphics[width=0.32\linewidth]{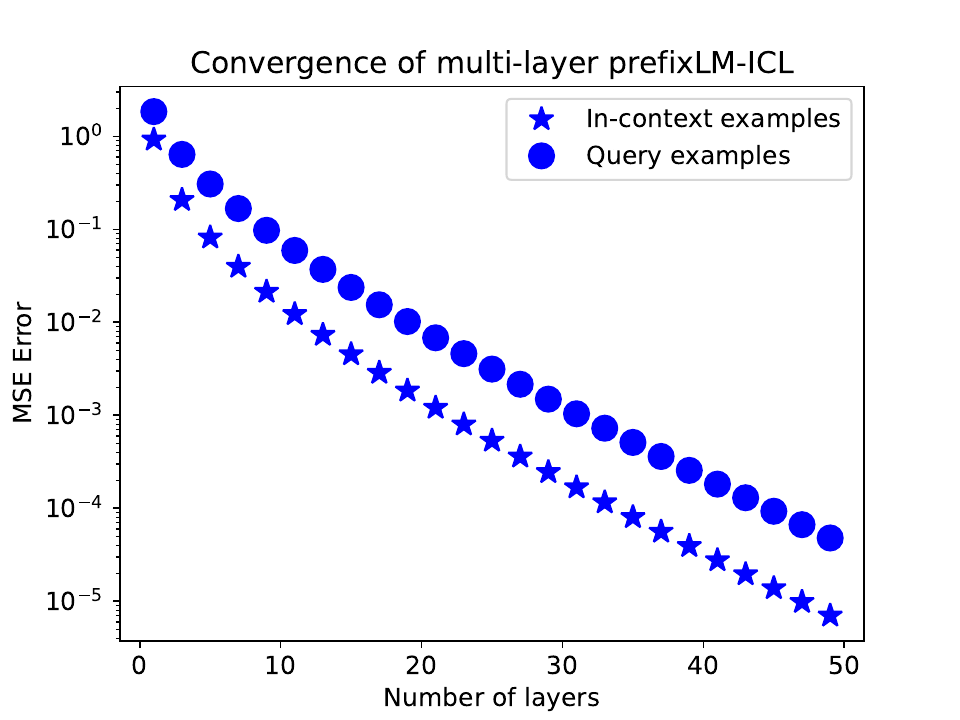}
  \includegraphics[width=0.32\linewidth]{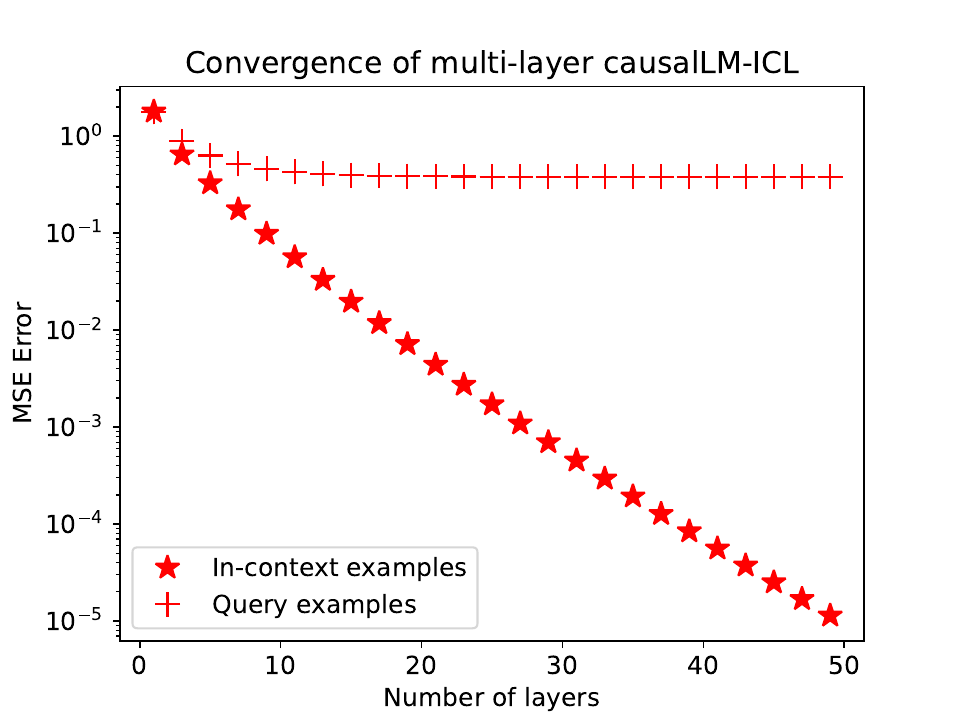}
  \includegraphics[width=0.32\linewidth]{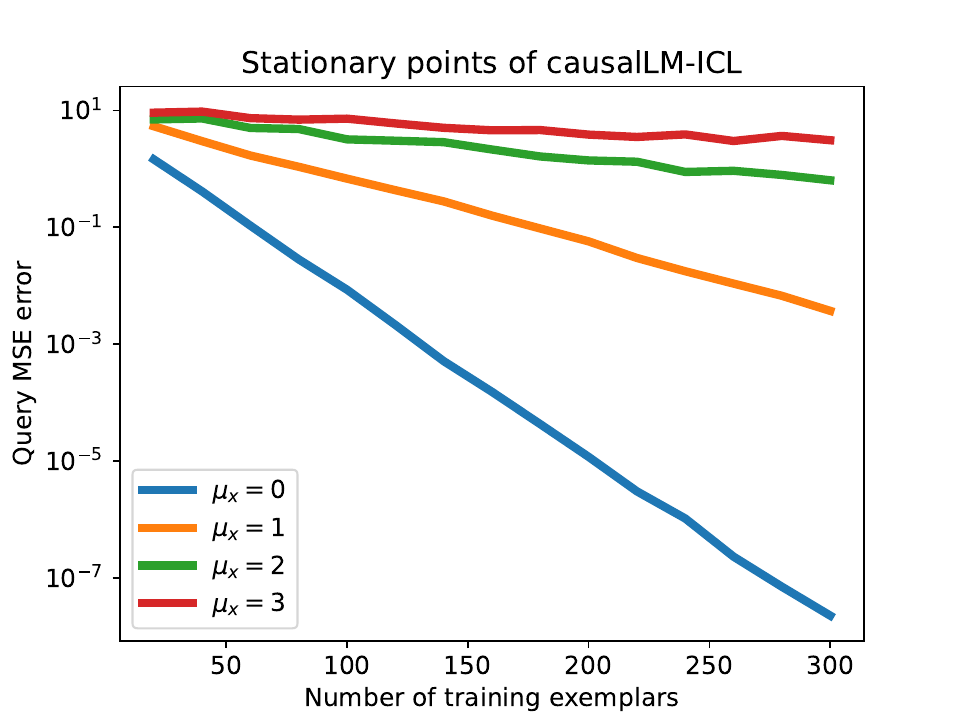}
   \caption{Left/Middle: the MSE on in-context examples and query examples of multi-layer LSA-based prefixLM/causalLM-ICLs with 40 in-context training examples. Right: the query MSE of causalLM-ICL's stationary points (per Proposition~\ref{prop:causal-icl}) using up to 300 in-context examples.}
   \label{fig:icl-convergence}
\end{figure*}

\subsection{Standard transformers on synthetic tasks}
\label{sec:transformer-synthetic}
Previous experiments provided a proof of concept verification of the propositions from Section \ref{sec:theory-ml-icl}. Next we verify if a standard softmax transformer-based prefixLM and causalLM ICL exhibit similar differences on various types of synthetic tasks including linear regression, non-linear regression and multiclass classification.

All three tasks used 16-dim inputs with $\xb \sim \Ucal(-1, 1)^{16}$ and $\wb \sim \Ncal(0, \Ib)$. For non-linear regression, we apply a sigmoid activation on the logit such that $y = \text{sigmoid}(\wb \xb)$; and for multiclass classification, we randomly generate three $\wb_c \sim \Ncal(0, \Ib)$, and assign labels based on $y = \argmax_c \cbr{\wb_c \xb}$.
\begin{table}
    \centering
    \begin{tabular}{c|c c c }
                 & LR  & N-LR & MC \\ \hline
        PrefixLM-SL & 8.6e-3  & 1.5e-4 & {\bf 24.1} \\
        CausalLM-SL & 1.9e-1 & 2.7e-3 & 27.0 \\\hline
        PrefixLM-UL & {\bf 2.5e-3} & {\bf 9.0e-5} & 27.6 \\
        CausalLM-UL & 1.6e-2  & 2.9e-3 & 32.1 \\
    \end{tabular}
    \caption{The test query errors of the unshared-layer (UL) and sharing-layer (SL) transformer-ICLs on linear regression (LR), non-linear regression (NLR), and multiclass classification (MC) tasks. Both regression tasks report mean squared errors; and the MC task reports the classification error.}
    \label{tab:share_vs_noshare}
\end{table}
We trained a few 24-layer transformers containing 128 hidden units and 2 heads. Besides of the comparisons of prefixLM and causalLM, we also compare the transformers with or without sharing layers (SL vs UL). In particular, the sharing-layer transformer can be considered a recurrent system~\citep{dehghani2018universal} where the dynamic is continual along the layers and a stationary point may exist given infinite number of layers, which makes it closer to our constructed LSA. 

The ICL training dataset contains 64,000 training sequences. Each sequence contains 40 in-context examples and 20 queries, where queries are independent of each other similar to Section~\ref{sec:lsa-lr}. The transformers are trained with batch size 64 for 100 epochs.
More details of the hyper-parameters of the experiments are provided in Appendix \ref{sec:exp_details}. 

We evaluate the ICL performance using 64 holdout test sequences and report the test errors on the query examples. The results are summarized in Table~\ref{tab:share_vs_noshare}. We find that both prefixLM-SL and prefixLM-UL significantly outperform their counterparts of causalLM in all cases. As a side note, transformer-SL appears to outperform transformer-UL in the classification tasks, which indicates the overfitting problem of the latter due to over-parameterization. In addition, we also add probes 
at the output of each SL-transformer layer to visualize the test errors of intermediate layers in Figure \ref{fig:transformer-general}. Comparing Figure \ref{fig:transformer-general} and Figure \ref{fig:icl-convergence} (left/middle) reveals some similarities. Although the test query errors of causalLM decay in roughly the same rate as the ones of prefixLM in earlier layers, the decays become much slower in later layers possibly due to the nature of its non-optimal stationary points. These results suggest that the title argument of the paper also holds beyond LSA-based transformers and linear regression.

\begin{figure*}
  \centering
  \includegraphics[width=0.32\linewidth]{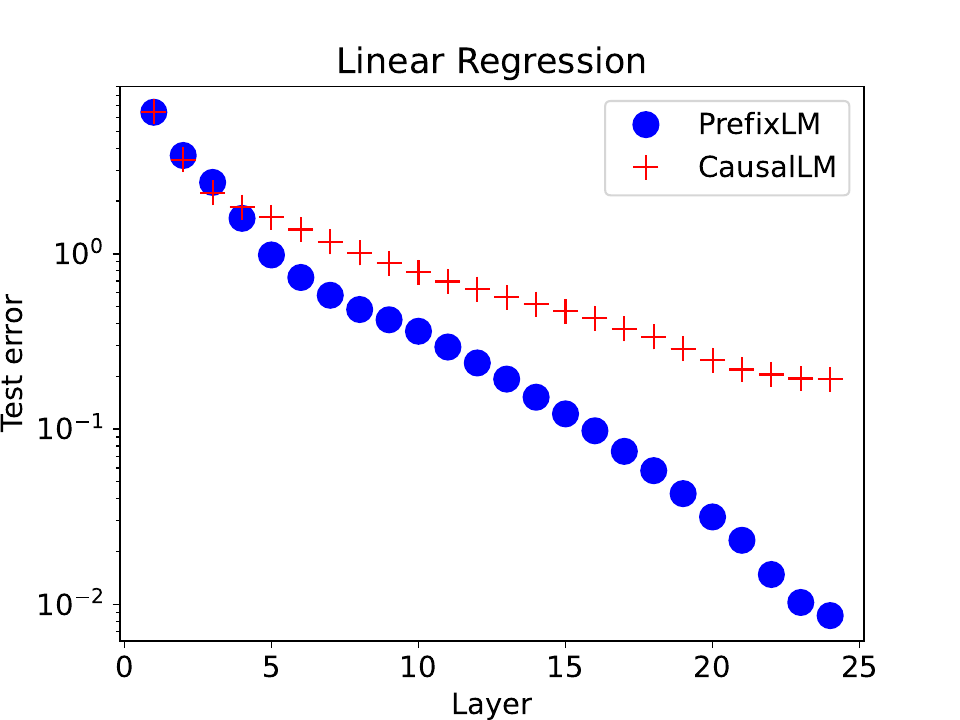}
  \includegraphics[width=0.32\linewidth]{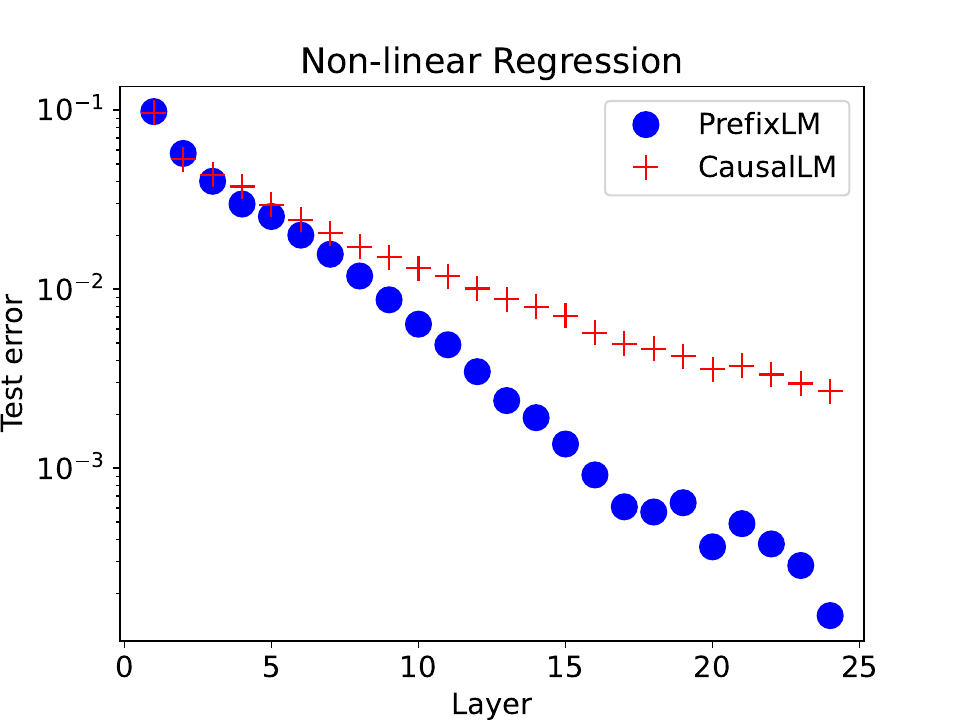}
  \includegraphics[width=0.32\linewidth]{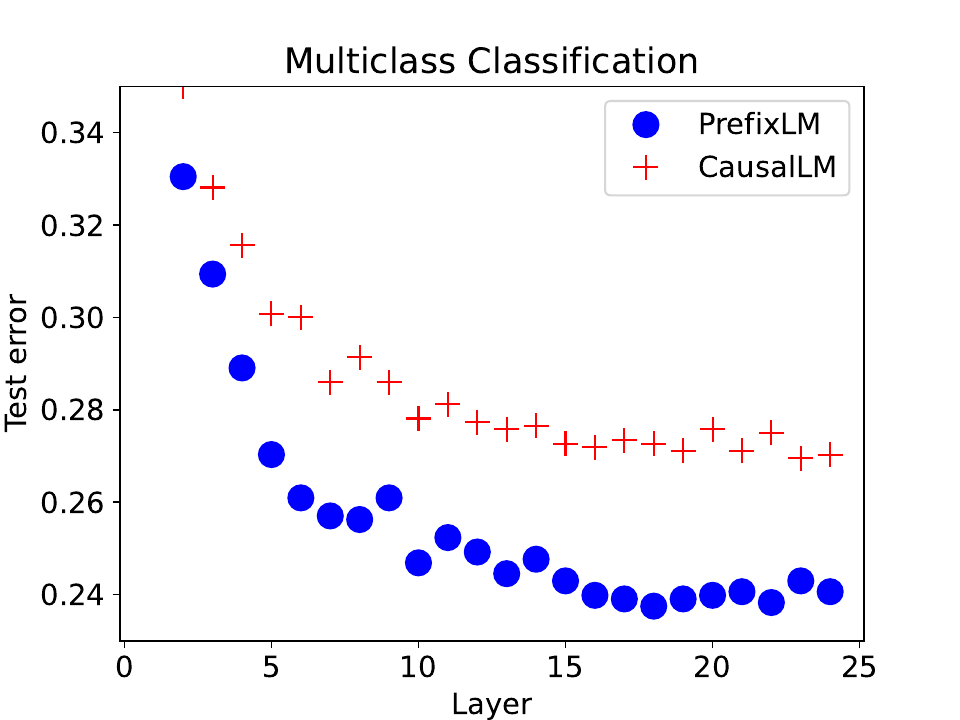}
   \caption{The test query errors of the 24-layer SL-transformers based prefixLM/causalLM-ICLs on linear regression (left), non-linear regression (middle), and multiclass classification (right).}
   \label{fig:transformer-general}
\end{figure*}

\subsection{ICL on large language models}
\label{sec:flan-t5x}
In order to compare the ICL performance of causalLM and prefixLM in a large language model setting, we conduct experiments using the publicly available T5 family of models~\citep{roberts2022t5x}. Note that the existing public T5X~\footnote{\url{https://github.com/google-research/t5x}} checkpoints are all based on EncDec models, which are similar to prefixLM. Thus, it would be unfair and unnatural to compare with causalLM by simply replacing the bidirectional attention of the encoder to the causal attention during the finetuning stage. To make a more reasonable comparison, we reran the pretraining stages of T5 on the C4 corpus~\citep{2020t5} from a random initialization point using a span corruption objective, but in the DecoderOnly setting. 
Moreover, for each size (from Base, Large and XL) of the models, we pretrained two checkpoints, one with prefixLM and the other with causalLM, each for 1M steps using the same T5 pretraining recipe. After pretraining, we use the FLAN recipe~\citep{chung2022scaling} to finetune each checkpoint (40k steps for Base, 20k steps for Large and XL) with its pretrained attention mask and evaluate the ICL capability of the finetuned models on two benchmarks: MMLU~\citep{hendrycks2020measuring} and BBH~\citep{suzgun2022challenging}.

Table~\ref{tab:t5x} shows that for all three sizes of T5X DecoderOnly models, the MMLU and BBH accuracies of prefixLM surpasses that of causalLM consistently and such gap widens as the size of the model becomes larger. This result empirically verifies that our conjecture generalizes to the practical case. 
We supply additional empirical evidence on state-of-the-art models in Appendix~\ref{sec:palm2} and~\ref{sec:palix}.

\begin{table}
    \centering
    \begin{tabular}{c|c c c | c c c}
                 & & MMLU & & & BBH & \\ 
                 & Base & Large & XL & Base & Large & XL \\ \hline
        PrefixLM & {\bf 28.8}  & {\bf 32.0} & {\bf 39.5} & {\bf 27.4}  & {\bf 32.2} & {\bf 35.8} \\
        CausalLM & 28.0  & 26.9 & 30.5 & 24.8 & 29.8 & 32.0 \\
    \end{tabular}
    \caption{The averaged test query accuracies on 5-shot MMLU (57 tasks) and 3-shot BBH (23 tasks) with FLAN-finetuned T5 DecoderOnly prefixLM/causalLM checkpoints.}
    \label{tab:t5x}
\end{table}

\section{Conclusion}
In this paper, we analyzed the convergence properties of two types of widely-used transformer-based language models (causalLM and prefixLM), during in-context learning. Using a simplified LSA attention in a linear regression setting, we proved that both LM types converge to their stationary points in linear rates, but that their stationary points have significantly different properties. In particular, the stationary points of prefixLM coincides with the optimal least square solution; while the ones of causalLM is equivalent to the weights of an online learning system, that is not guaranteed to converge to the optimal solution. Our experiments verified the above theoretical results, and also empirically extend the findings to general transformer on non-linear regression as well as classification tasks. Finally, we compare causalLM and prefixLM on a few large language models and find that prefixLM also consistently wins over causalLM in practical few-shot tasks.

\clearpage

\section*{Acknowledgements}
We want to specially thank Xinhua Zhang for helpful discussions and references about online learning. We also thank Yi Tay for comments regarding the PaLM2 checkpoints.

\bibliography{custom}
\bibliographystyle{iclr2024_conference}

\clearpage
\appendix

{\centering
	{\bf \Large CausalLM is not optimal for in-context learning\par}\vspace{2ex}
 {\bf \large Appendices\par}\vspace{4ex}}

\section{Related Work}
\label{sec:related}
Ever since GPT-3~\citep{brown2020language} exhibited its in-context learning abilities in various language inference and translation tasks, there has been tremendous interest in understanding the mechanics behind In-Context Learning (ICL).
Currently, there are two main camps of thought that try to explain ICL: 
(1) the representation camp, which views ICL behavior as a topic model that extracts relevant memories based on the topic of the context~\citep{xie2021explanation,min2022rethinking} - these works support this view with the findings that in-context learner sometimes behaved similarly even when the label of the training examples were permuted~\citep{min2022rethinking}. 
(2) the algorithmic camp, which holds that LLMs learns to implement a learning algorithm~\citep{garg2022can,akyurek2022learning,von2023transformers} and then run it during ICL - 
these works usually propose a construction of the transformer parameters and show that it can solve certain simple tasks (e.g. linear regression), then empirically verify that transformers track the behavior of the algorithm of interest.

Moreover, recent studies of large-scale data and model~\citep{wei2023larger} discovered that large language models seem to exhibit certain emergent behavior, where, ICL is memory-based on small-to-medium sized models or data, but becomes more algorithm-based on larger model and data. For example, ~\citep{wei2023larger} showed that a large language model is able to respond accordingly to the flipped label in in-context examples, opposing the findings of~\citep{min2022rethinking}.

Since most ICL applications only involve few shots of context examples, it seems reasonable to conjecture that the memory of a \emph{deep} representation and a \emph{shallow} predictor algorithm may co-exist in contributing the in-context learning capabilities. Since the representation learning of large language models have been universally acknowledged, it is more interesting to investigate how transformer learns to in-context learn shallow predictors with few-shot examples.

Focusing on work from the algorithmic camp, we note that~\citep{garg2022can} were the first to suggest using linear regression to study in-context learning. 
The authors empirically found that a 12-layer transformer is able to achieve similar results as a least-square solver on a 20-dim linear regression problem with around 20 in-context examples. Beyond linear regression, they also found that transformers can in-context learn a few other classes of shallow predictors, including two-layer Relu networks.

Probably the first formal theoretical investigation of the linear regression in-context learners is \citep{akyurek2022learning}. They first showed that a transformer layer can approximately conduct four basic operations: mov, mul, div, aff. They then cleverly combined these four operations and showed that
%when the input is formulated as 
%\begin{align*}
%    \Eb = 
%    \begin{pmatrix}
%\xb_1 & {\bf 0} & \cdots & \xb_n & {\bf 0} & \xb_{query}\\
%0 & y_1 & \cdots & 0 & y_n & 0
%\end{pmatrix},
%\end{align*}
a gradient descent step of linear regression can be implemented with a 4-head 8-layer transformer with $O(d)$ hidden units, where $d$ is the dimension of the inputs $\xb$. 
Despite their novel construction, the result itself provides only a loose upper bound on the model size (or depth) that is required for simulating linear regression within a transformer - for example,~\citep{von2023transformers} reported that a 2 or 5-layer transformer already achieves significantly better results than a single-step gradient descent for linear regression.

Because of the significant discrepancy between the construction of \citep{akyurek2022learning} and the empirical results, the one-layer LSA construction of~\citep{von2023transformers} appears to be more appealing and matches the experimental results better. Moreover, a most recent work by~\citep{zhang2023trained} used gradient flow to prove that by initializing $\wb^{(0)}=0$, such matrix constructions can indeed be learned by an LSA transformer. This is why our paper follows this construction and studies its multi-layer convergence properties with different types of attention (prefixLM vs causalLM).

In terms of the comparison between prefixLM and causalLM, such research work can be traced back as early as~\citep{raffel2020exploring}, where they showed prefixLM outperforms causalLM in varieties of NL tasks. Later, UL-2~\citep{tay2022unifying} proposed to mix prefixLM and span corruption objectives, and found it to be more efficient than the causalLM objective alone. It was also shown in~\citep{chung2022scaling}, that U-PaLM (a UL2-finetuning PaLM) outperforms PaLM (causalLM only) in various ICL tasks. Indeed, for the reasons above, some of the latest models have included prefixLM objectives in the pretraining mix (for example PaLM-2 by~\cite{palm2}). On the other hand, prominent models such as Flamingo as well as the ones in the GPT-family are still based on the causalLM structure, so the comparison between prefixLM and causalLM remains important and relevant. Furthermore, all previous studies were done in an empirical manner, whereas we set out to explain their differences from a theoretical perspective and back the theory with empirical evidence. While we are not the first to follow this path, our work is the first to provide a theoretical justification for the advantage of prefixLM over causalLM in a multi-layer transformer ICL setting by analyzing their theoretical convergence properties.

\section{Proofs}
\label{sec:proof}
In this section, we provide proofs of the propositions introduced in Section~\ref{sec:ml-icl} and Section~\ref{sec:theory-ml-icl}.

\paragraph{Proposition 1}
For a multi-layer LSA satisfying \eqref{eq:constructed_lsa} with $\wb^{(0)} = 0$, if its input $\Zb$ is formatted as~\eqref{eq:icl_input}, then its $l$-th layer output is $\zb_j^{(l)} = (\xb_j^\top, \delta_j^{(l)})^\top$, where $\delta_j^{(l)} = y_j - \wb^{(l)} \xb_j$ and $\wb^{(l)}$ is the weight from the $l$-th step gradient descents as in~\eqref{eq:prefix_update_w}.

\begin{proof}
Plugging in $\Kb$, $\Qb$, $\Pb$ and $\Vb$ of \eqref{eq:constructed_lsa} with $\wb^{(0)}=0$ into \eqref{eq:lsa}, we have
{\small
\begin{align*}
   &\begin{pmatrix}\xb_j \\ \delta_j^{(l)} \end{pmatrix} = \begin{pmatrix}\xb_j \\ \delta_j^{(l-1)} \end{pmatrix} + \frac{\eta}{n} \begin{pmatrix} {\bf 0}_{d \times d} & {\bf 0} \\ 0 & -1\end{pmatrix} \cdot \\
   &\;\;\rbr{\sum_{i=1}^n \begin{pmatrix}\xb_i \\ \delta_i^{(l-1)} \end{pmatrix} (\xb_i^\top, \delta_i^{(l-1)}) \begin{pmatrix} \Ib_{d \times d} & {\bf 0} \\ 0 & 0\end{pmatrix} \begin{pmatrix}\xb_j \\ \delta_j^{(l-1)} \end{pmatrix}} \\
   &\qquad\;\; = \begin{pmatrix}\xb_j \\ \delta_j^{(l-1)} \end{pmatrix} - \frac{\eta}{n} \sum_{i=1}^n \begin{pmatrix}{\bf 0} \\ \delta_i^{(l-1)} \end{pmatrix} \xb_i^\top \xb_j.
\end{align*}
}
It is easy to see that $\zb_j$ never changes its first $d$-dimension corresponding to $\xb_j$. Therefore, we can simplify the above equation and focus only on the last coordinate $\delta_j^{(l)}$, where we have
\begin{align}
\delta_j^{(l)} =& \; \delta_j^{(l-1)} - \frac{\eta}{n} \sum_{i=1}^n \delta_i^{(l-1)} \xb_i^\top \xb_j, \label{eq:prefix_update_delta_app}
\end{align}
with $\delta_j^{(0)} = y_j$. Defining $\tilde{y}_j^{(l)} = y_j - \delta_j^{(l)}$ and rearranging \eqref{eq:prefix_update_delta_app}, we obtain $\tilde{y}_j^{(0)} =0$ and
\begin{align}
\tilde{y}_j^{(l)} =& \; \tilde{y}_j^{(l-1)} + \frac{\eta}{n} \sum_{i=1}^n (y_i - \tilde{y}_i^{(l-1)}) \xb_i^\top \xb_j. \label{eq:prefix_update_y_app}
\end{align}
Next we prove $\tilde{y}_j^{(l)} = \wb^{(l)} \xb_j$ by induction.
Since $\wb^{(0)} = 0$, it is clear that $\tilde{y}_j^{(0)} = \wb^{(0)} \xb_j = 0$ for all $j$.

If $\tilde{y}_j^{(l-1)} = \wb^{(l-1)} \xb_j$ for all $j$, then
\begin{align*}
    \tilde{y}_j^{(l)} =& \; \tilde{y}_j^{(l-1)} + \frac{\eta}{n} \sum_{i=1}^n (y_i - \tilde{y}_i^{(l-1)}) \xb_i^\top \xb_j \\
    = & \wb^{(l-1)} \xb_j + \frac{\eta}{n} \sum_{i=1}^n (y_i - \wb^{(l-1)} \xb_i) \xb_i^\top \xb_j \\
    = & \rbr{\wb^{(l-1)} + \sum_{i=1}^n (y_i - \wb^{(l-1)} \xb_i) \xb_i^\top} \xb_j \\
    = & \wb^{(l)} \xb_j.
\end{align*}
\end{proof}

\paragraph{Proposition 2} 
For a multi-layer causalLM-LSA satisfying \eqref{eq:constructed_lsa} with $\wb^{(0)} = 0$, if its input $\Zb$ is formatted as~\eqref{eq:icl_input}, then its $l$-th layer output is $\zb_j^{(l)} = (\xb_j^\top, \delta_j^{(l)})^\top$, where $\delta_j^{(l)} = y_j - \wb_j^{(l)} \xb_j$ and $\wb_j^{(l)}$ follow~\eqref{eq:causal_update_w}.
\begin{proof}
Plugging in $\Kb$, $\Qb$, $\Pb$ and $\Vb$ of \eqref{eq:constructed_lsa} with $\wb^{(0)}=0$ into \eqref{eq:causal_lsa}, we have
\begin{align*}
    \delta_j^{(l)} =& \; \delta_j^{(l-1)} - \frac{\eta}{n} \sum_{i=1}^j \delta_i^{(l-1)} \xb_i^\top \xb_j\\
\tilde{y}_j^{(l)} =& \; \tilde{y}_j^{(l-1)} + \frac{\eta}{n} \sum_{i=1}^j (y_i - \tilde{y}_i^{(l-1)}) \xb_i^\top \xb_j
\end{align*}
with $\tilde{y}_j^{(l)} = y_j - \delta_j^{(l)}$. 
Next we prove $\tilde{y}_j^{(l)} = \wb_j^{(l)} \xb_j$ by induction.
Since $\wb_j^{(0)} = 0$, it is clear that $\tilde{y}_j^{(0)} = \wb_j^{(0)} \xb_j = 0$ for all $j$.

If $\tilde{y}_j^{(l-1)} = \wb_j^{(l-1)} \xb_j$ for all $j$, then
\begin{align*}
    \tilde{y}_j^{(l)} =& \; \tilde{y}_j^{(l-1)} + \frac{\eta}{j} \sum_{i=1}^n (y_i - \tilde{y}_i^{(l-1)}) \xb_i^\top \xb_j \\
    = & \wb_j^{(l-1)} \xb_j + \frac{\eta}{n} \sum_{i=1}^j (y_i - \wb_i^{(l-1)} \xb_i) \xb_i^\top \xb_j \\
    = & \rbr{\wb_j^{(l-1)} + \sum_{i=1}^n (y_i - \wb_i^{(l-1)} \xb_i) \xb_i^\top} \xb_j \\
    = & \wb_j^{(l)} \xb_j.
\end{align*}
\end{proof}

\paragraph{Proposition 3} 
If $\wb^{(l)}$ follows the iterative updates of \eqref{eq:prefix_update_w}, then there exists a stationary point $\wb^*$ with coefficients satisfying:
\begin{align*}
   \yb \xb^\top = \wb^* \Xb \Xb^\top,
\end{align*}
where $\yb = (y_1, \ldots, y_n)$ and $\Xb = (\xb_1, \ldots, \xb_n)$. Furthermore, the iterative weights $\wb^{(l)}$ converges to the stationary point $\wb^*$ with linear rate of convergence:
\begin{align*}
   \wb^{(l)} - \wb^* = (\wb^{(l-1)} - \wb^*)(\Ib - \frac{\eta}{n} \Xb \Xb^\top ).
\end{align*}

\begin{proof}
From \eqref{eq:prefix_update_w}, we have
\begin{align*}
\wb^{(l)} =& \; \wb^{(l-1)} + \frac{\eta}{n} \underbrace{\sum_{i=1}^n (y_i - \wb^{(l-1)} \xb_i) \xb_i^\top}_{(*)}.
\end{align*}
The stationary point must satisfy $(*) = 0$. Written in vectorized form, we have
\begin{align}
    \yb \Xb^\top = \wb^* \Xb \Xb^\top. \label{eq:prefix_stationary}
\end{align}
Now plugging \eqref{eq:prefix_stationary} back to \eqref{eq:prefix_update_w}, we have
{\small
\begin{align*}
   \wb^{(l)} = \wb^{(l-1)} + \frac{\eta}{n} \rbr{\wb^* \Xb \Xb^\top - \ab \wb^{(l-1)} \Xb \Xb^\top}, 
\end{align*}}
which can be reorganized to
\begin{align*}
   \wb^{(l)} - \wb^* = (\wb^{(l-1)} - \wb^*)(\Ib - \frac{\eta}{n} \Xb\Xb^\top ).
\end{align*}
\end{proof}

\paragraph{Proposition 4}
If $\wb_j^{(l)} = \sum_{i=1}^j a_{i,j}^{(l)} \xb_i^\top$ follows the iterative updates of \eqref{eq:causal_update_w}, then 
\begin{align*}
a_{i,j}^{(l)} = a_{i,i}^{(l)} \equiv a_i^{(l)} \;\;\;\; \forall j \ge i,
\end{align*}
and there exists stationary points $\wb_j^* = \sum_{i=1}^j a_i^* \xb_i^\top$ (for $j \in 1, \ldots, n$) with coefficients from $\ab^* = (a^*_1, \ldots, a^*_n)$ that satisfy $\yb = \ab^* \Tb$, where
\begin{align*} 
\Tb = \begin{pmatrix}
\xb_1^\top \xb_1 & \xb_1^\top \xb_2 & \cdots & \xb_1^\top \xb_n \\
0 & \xb_2^\top \xb_2 & \cdots & \xb_2^\top \xb_n \\
\vdots & \vdots & \ddots & \vdots \\
0 & 0 & \cdots & \xb_n^\top \xb_n 
\end{pmatrix}.
\end{align*}
Furthermore, the coefficients $\ab^{(l)}$ converges to the stationary point $\ab^*$ with linear rate of convergence:
\begin{align*}
   \ab^{(l)} - \ab^* = (\ab^{(l-1)} - \ab^*)(\Ib - \frac{\eta}{n} \Tb ).
\end{align*}

\begin{proof}
First notice that according to \eqref{eq:causal_update_w}, we have
\begin{align*}    
\wb_j^{(l)} =& \; \wb_j^{(l-1)} + \frac{\eta}{n} \sum_{i=1}^j (y_i - \wb_i^{(l-1)} \xb_i) \xb_i^\top \\
=& \sum_{i=1}^j \rbr{a_{i,j}^{(l-1)} + \frac{\eta}{n}(y_i - \wb_i^{(l-1)} \xb_i) } \xb_i^\top
\end{align*}
or
\begin{align*}
a_{i,j}^{(l)} = a_{i,j}^{(l-1)} + \frac{\eta}{n}(y_i - \wb_i^{(l-1)} \xb_i) \;\;\;\; \forall j \ge i.
\end{align*}
Since $a_{i,j}^{(0)} = 0$, and the above update is the same for any $j$ given any $i$, then it is obvious by induction that
\begin{align*}
    a_{i,j}^{(l)} = a_{i,i}^{(l)} \equiv a_i^{(l)} \;\;\;\; \forall j \ge i.
\end{align*}
Therefore, we can simplify $\wb_j^{(l)} = \sum_{i=1}^j a_i^{(l)} \xb_i^\top$.

Now plugging into \eqref{eq:causal_update_w}, we have
{\small
\begin{align*}
&\sum_{i=1}^j a_i^{(l)} \xb_i^\top \\
=& \sum_{i=1}^j a_i^{(l-1)} \xb_i^\top + \frac{\eta}{n} \sum_{i=1}^j (y_i - \sum_{k=1}^i a_k^{(l-1)} \xb_k^\top \xb_i) \xb_i^\top\\
=& \sum_{i=1}^j \rbr{a_i^{(l-1)} + y_i - \frac{\eta}{n} \rbr{\sum_{k=1}^i a_k^{(l-1)} \xb_k^\top \xb_i}} \xb_i^\top,
\end{align*}
}
which is equivalent to
\begin{align}
a_i^{(l)} = a_i^{(l-1)} + \frac{\eta}{n} \underbrace{\rbr{y_i - \sum_{k=1}^i a_k^{(l-1)} \xb_k^\top \xb_i}}_{(*)}. \label{eq:causal_update_a}
\end{align}
The stationary points satisfy $(*) = 0$, which gives
\begin{align*}
    y_1 =& a_1^* \xb_1^\top \xb_1 \\
    y_2 =& a_1^* \xb_1^\top \xb_2 + a_2^* \xb_2^\top \xb_2 \\
    &\ldots \\
    y_n =& a_1^* \xb_1^\top \xb_n + \ldots + a_n^* \xb_n^\top \xb_n,
\end{align*}
or in the vectorized form $\yb = \ab^* \Tb$, where
\begin{align*} 
\Tb = \begin{pmatrix}
\xb_1^\top \xb_1 & \xb_1^\top \xb_2 & \cdots & \xb_1^\top \xb_n \\
0 & \xb_2^\top \xb_2 & \cdots & \xb_2^\top \xb_n \\
\vdots & \vdots & \ddots & \vdots \\
0 & 0 & \cdots & \xb_n^\top \xb_n 
\end{pmatrix}.
\end{align*}

Now plugging in $\yb = \ab^* \Tb$ back to \eqref{eq:causal_update_a} and vectorize it, yields
\begin{align*}
    \ab^{(l)} = \ab^{(l-1)} + \frac{\eta}{n} \rbr{\ab^{*} \Tb - \ab \Tb}, 
\end{align*}
which can be reorganized to 
\begin{align*}
   \ab^{(l)} - \ab^* = (\ab^{(l-1)} - \ab^*)(\Ib - \frac{\eta}{n} \Tb ).
\end{align*}
\end{proof}

\paragraph{Proposition 5}
Assuming that $\wb_j^*$ is the stationary points obtained in Proposition \ref{prop:causal-icl}, then 
\begin{align*}
    \wb_{j+1}^* = \wb_j^* - \frac{1}{\|\xb_{j+1}\|_2^2} \nabla_{\wb_j^*} L_j(\wb_j^*).
\end{align*}
\begin{proof}
Recall the online learning system with a sequence of data-label pairs $(\xb_j, y_j)$ has the following online loss and its gradient at the $j$-th step,
\begin{align*}
    L_j(\wb_j) &= \frac{1}{2} (\wb_j \xb_{j+1} - y_{j+1})^2, \\
    \nabla_{\wb_j} L_j(\wb_j) &= (\wb_j \xb_{j+1} - y_{j+1}) \xb_{j+1}^\top.
\end{align*}
According to Proposition \ref{prop:causal-icl}, we have $\yb = \ab^* \Tb$, which gives
\begin{align}
    y_{j+1} =& a_1^* \xb_1^\top \xb_{j+1} + \ldots + a_j^* \xb_j^\top \xb_{j+1} \nonumber\\
   &\; + a_{j+1}^* \xb_{j+1}^\top \xb_{j+1} \nonumber\\
   =& \wb_j^* \xb_{j+1} + a_{j+1}^* \xb_{j+1}^\top \xb_{j+1} \label{eq:proof4-inter1}
\end{align}
where the last equation is due to $\wb_j^* = \sum_{i=1}^j a_i^* \xb_i^\top$.

Since $\wb_j^* = \sum_{i=1}^j a_i^* \xb_i^\top$, we have
\begin{align*}
    \wb_{j+1}^* =& \wb_j^* + a_{j+1}^* \xb_{j+1}^\top \\
    =& \wb_j^* + \frac{1}{\|\xb_{j+1}\|_2^2} \rbr{a_{j+1}^* \xb_{j+1}^{\top} \xb_{j+1}}\xb_{j+1}^\top\\
    =& \wb_j^* - \frac{1}{\|\xb_{j+1}\|_2^2} (\wb_j^* \xb_{j+1} - y_{j+1}) \xb_{j+1}^\top\\
    =& \wb_j^* - \frac{1}{\|\xb_{j+1}\|_2^2} \nabla_{\wb_j^*} L_j(\wb_j^*)
\end{align*}
where the third equation is because of \eqref{eq:proof4-inter1}.
\end{proof}

\begin{comment}
\begin{corollary}
\label{coro:condition_number}
The condition numbers of the matrices $\Tb$ and $\Xb^\top \Xb$ satisfy $\kappa(\Tb) \le \kappa(\Xb^\top \Xb)$,
where $\Xb = (\xb_1, \ldots, \xb_n)$ and $\Tb = triu(\Xb^\top \Xb)$. 
\end{corollary}
\begin{proof}
Since $\Tb$ is an upper triangular matrix, its eigenvalues are its diagonal. Therefore, its condition number is
\begin{align*}
    \kappa(\Tb) = \frac{\max_i \|\xb_i\|_2^2}{\min_i \|\xb_i\|_2^2}.
\end{align*}
On the other hand, the maximum eigenvalue of $\Xb^\top \Xb$
\begin{align*}
    \sigma_{max} (\Xb^\top \Xb) =& \max_{\vb} \frac{\vb^\top \Xb^\top \Xb \vb}{\vb^\top \vb} \\
    \ge & \max_i \frac{\eb_i^\top \Xb^\top \Xb \eb_i}{\eb_i^\top \eb_i} = \max_i \|\xb_i\|_2^2
\end{align*}
where $e_i$ represents the $i$-th index vector. Similarly, 
\begin{align*}
    \sigma_{min} (\Xb^\top \Xb) \le \min_i \|\xb_i\|_2^2
\end{align*}
Therefore, 
\begin{align*}
    \kappa(\Xb^\top \Xb) = \frac{\sigma_{max} (\Xb^\top \Xb)}{\sigma_{min} (\Xb^\top \Xb)} \ge \kappa(\Tb).
\end{align*}
\end{proof}    
\end{comment}

\section{Multi-layer LSA construction with non-zero w(0) }
\label{sec:general_w0}
In this section, we introduce the proposition that connects a multi-layer LSA following the construction of \eqref{eq:constructed_lsa} but with non-zero $\wb^{(0)}$ and the multi-step gradient descents of linear regression.
\begin{proposition}
For a multi-layer LSA satisfying the construction \eqref{eq:constructed_lsa}, if its input $\Zb$ is formatted as~\eqref{eq:icl_input}, then its $l$-th layer output is $\zb_j^{(l)} = (\xb_j^\top, \delta_j^{(l)})^\top$, where $\delta_j^{(l)} = y_j - (\wb^{(l)} - \wb^{(0)}) \xb_j$ and $\wb^{(l)}$ is the $l$-th updated weight from the gradient descents update rule in~\eqref{eq:prefix_update_w}.
\end{proposition}
\begin{proof}
Plugging in $\Kb$, $\Qb$, $\Pb$ and $\Vb$ of \eqref{eq:constructed_lsa} into \eqref{eq:lsa}, we have
\begin{align}
\delta_j^{(l)} =& \; \delta_j^{(l-1)} - \frac{\eta}{n} \sum_{i=1}^n \rbr{\delta_i^{(l-1)} - \wb^{(0)}\xb_i }\xb_i^\top \xb_j, \label{eq:prefix_update_delta_general_w0}
\end{align}
with $\delta_j^{(0)} = y_j$. Defining $\tilde{y}_j^{(l)} = y_j - \delta_j^{(l)} + \wb^{(0)}\xb_j$ and rearranging the \eqref{eq:prefix_update_delta_general_w0}, we obtain $\tilde{y}_j^{(0)} =0$ and
\begin{align*}
\tilde{y}_j^{(l)} =& \; \tilde{y}_j^{(l-1)} + \frac{\eta}{n} \sum_{i=1}^n (y_i - \tilde{y}_i^{(l-1)}) \xb_i^\top \xb_j.
\end{align*}
Then it is easy to prove $\tilde{y}_j^{(l)} = \wb^{(l)} \xb_j$ by induction, similar to the proof of Proposition~\ref{prop:muicl}.
\end{proof}

\section{CausalLM with attention-length-based coefficients}
\label{sec:causallm2}
Since there are $j$ terms in the summation of \eqref{eq:causal_update_w}, another reasonable update for causalLM would be
\begin{align}    
\wb_j^{(l)} =& \; \wb_j^{(l-1)} + \frac{\eta}{j} \sum_{i=1}^j (y_i - \wb_i^{(l-1)} \xb_i) \xb_i^\top, \label{eq:causal_update_w_2}
\end{align}
which we call causalLM2. For causalLM2, we have the following proposition. 
\begin{proposition}
\label{prop:causal-icl-2}
If $\wb_j^{(l)} = \sum_{i=1}^j a_{i,j}^{(l)} \xb_i^\top$ follows the iterative updates of \eqref{eq:causal_update_w_2}, then 
\begin{align*}
a_{i,j}^{(l)} \equiv \frac{1}{j}a_i^{(l)} \;\;\;\; \forall j \ge i,
\end{align*}
and there exists stationary points $\wb_j^* = \frac{1}{j} \sum_{i=1}^j a_i^* \xb_i^\top$ (for $j \in 1, \ldots, n$) with coefficients from $\ab^* = (a^*_1, \ldots, a^*_n)$ that satisfy $\yb = \ab^* \Sbb$, where
\begin{align*} 
\Sbb = \begin{pmatrix}
\xb_1^\top \xb_1 & \frac{1}{2}\xb_1^\top \xb_2 & \cdots & \frac{1}{n}\xb_1^\top \xb_n \\
0 & \frac{1}{2}\xb_2^\top \xb_2 & \cdots & \frac{1}{n}\xb_2^\top \xb_n \\
\vdots & \vdots & \ddots & \vdots \\
0 & 0 & \cdots & \frac{1}{n}\xb_n^\top \xb_n 
\end{pmatrix}.
\end{align*}
Furthermore, the coefficients $\ab^{(l)}$ converges to the stationary point $\ab^*$ with the following rate of convergence:
\begin{align*}
   \ab^{(l)} - \ab^* = (\ab^{(l-1)} - \ab^*)(\Ib - \eta \Sbb ).
\end{align*}
\end{proposition}
The condition number $\kappa(\Sbb)$ is about $n/2$ greater than the one of $\kappa(\Tb)$, which makes causalLM2 converge much slower than causalLM. 

One can also prove that the stationary point of causalLM2 corresponds to the following online system with online loss and gradient at the $j$-th step,
\begin{align*}
    L_j(\tilde{\wb}_j) &= \frac{1}{2} (\tilde{\wb}_j \xb_{j+1} - y_{j+1})^2, \\
    \nabla_{\tilde{\wb}_j} L_j(\tilde{\wb}_j) &= (\tilde{\wb}_j \xb_{j+1} - y_{j+1}) \xb_{j+1}^\top,
\end{align*}
where $\tilde{\wb} = \frac{j}{j+1} \wb$.
\begin{proposition}
\label{prop:causal-ogd-2}
Assuming that $\wb_j^*$ is the stationary points obtained in Proposition \ref{prop:causal-icl}, then 
\begin{align*}
    \wb_{j+1}^* = \tilde{\wb}_j^* - \frac{1}{\|\xb_{j+1}\|_2^2} \nabla_{\tilde{\wb}_j^*} L_j(\tilde{\wb}_j^*).
\end{align*}
\end{proposition}
Since the step does not have $j^{-\nu}$ ($\nu > 0$) decay, such online system is not guaranteed to converge, therefore suffers the same problem as the original causalLM in Section \ref{sec:causallm}.

In Figure~\ref{fig:causallm2-convergence}, we plot the query MSE error of the stationary points of causalLM2-ICL with increasing number of in-context examples. We can see that the online system corresponding to causalLM2-ICL converges even slower than the ones of causalLM-ICL in Figure~\ref{fig:icl-convergence} right.
\begin{figure}[t]
  \centering
  \includegraphics[width=0.55\linewidth]{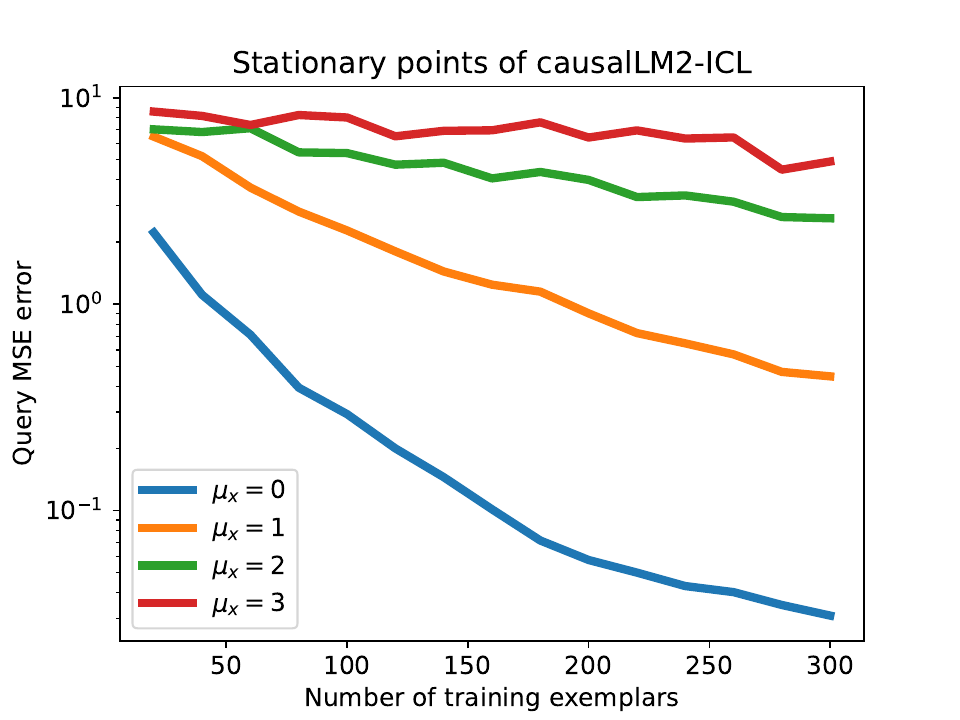}
   \caption{The test error on the stationary point of the causalLM2-ICL with up to 300 in-context examples.}
   \label{fig:causallm2-convergence}
\end{figure}

\section{Additional experimental details and results}
\label{sec:exp_details}
\subsection{Experiment settings for Section~\ref{sec:lsa-lr}}
In order to directly verify the theorem, we used the constructed LSA-based transformer, with $\Kb = \Qb = \begin{pmatrix} \Ib_{d \times d} & {\bf 0} \\ 0 & 0\end{pmatrix}$, $\Vb = \begin{pmatrix} {\bf 0}_{d \times d} & {\bf 0} \\ 0 & -1\end{pmatrix}$ and $\Pb = \frac{\eta}{n} \Ib$. Although not a trained transformer, it was recently proved in \citep{zhang2023trained} that a randomly initialized LSA-transformer does converge to such a construction. In addition, we did an ablation test of $\eta = \cbr{0.1, 0.2, 0.4, 0.8, 1.6, 3.2}$ and chose $\eta=1.6$ as it converges the fastest without any divergence problems. 

We randomly generated 64 sequences for ICL evaluation. For each sequence, we put the first 40 examples as the in-context examples and the last 200 examples as the query examples. The queries are independent of each other without attention. See Figure~\ref{fig:illu-attention} for an illustration of the transformer attention mask. Such multi-query design is for training efficiency purpose only and is equivalent to 200 sequences with the same $\wb$ and input examples $\xb_i$, but different one query per sequence. 
\begin{figure}[t]
  \centering
  \includegraphics[width=.55\linewidth]{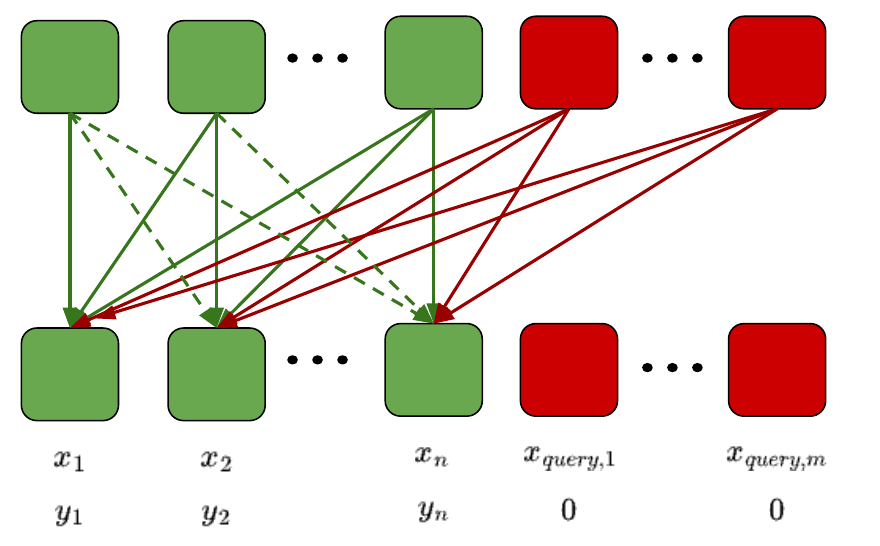}
   \caption{The illustration of the attention mask. Green arrows represent the attentions between in-context examples. The dashed arrows only applies for prefixLM. Red arrows represent the attentions from queries to in-context examples. The query examples should not attend to themselves because the inputs do not contain labels.}
   \label{fig:illu-attention}
\end{figure}

\subsection{Experiment settings for Section~\ref{sec:transformer-synthetic}}
In order to verify that our theorems can be qualitatively applied beyond LSA and linear regression, we conducted several experiments over various synthetic tasks using regular transformers.
We based our code from the repository of \citep{akyurek2022learning}\footnote{\url{https://github.com/google-research/google-research/tree/master/incontext}} and applied their default training hyperparameters of the code. We used a transformer of 24 layers with 128 hidden units and 2 heads. The FFN intermediate size is $4\times 128=512$. The learning schedule is based on cosine decay with base learning rate 1e-4, for 100 epochs.
In addition, since the target of the outputs of the in-context examples are 0 (see Fig.~\ref{fig:ml-icl}), we optionally add an additional L2 regularizer on the outputs of the in-context examples. See the comparison between the transformers with or without the L2-regularizer in Table~\ref{tab:l2_vs_nol2}. In Table~\ref{tab:share_vs_noshare} of the main paper, the reported numbers correspond to the SL-transformer with the L2 regularizer and the UL-transformer without the L2 regularizer.
Across all these settings prefixLM consistently beats causalLM as our theorem predicts.

\begin{table}
    \centering
    \begin{tabular}{c|c c c }
                 & LR  & N-LR & MC \\ \hline
        PrefixLM-SL-L2 & 8.6e-3  & 1.5e-4 & 24.1 \\
        CausalLM-SL-L2 & 1.9e-1 & 2.7e-3 & 27.0 \\\hline
        PrefixLM-SL-no-L2 & 6.7e-3  & 1.5e-4 & 24.5 \\
        CausalLM-SL-no-L2 & 5.0e-2 & 1.9e-3 & 30.5 \\\hline
        PrefixLM-UL-L2 & 7.6e-3  & 1.7e-4 & 26.7 \\
        CausalLM-UL-L2 & 4.4e-2 & 2.5e-3 & 30.4 \\\hline
        PrefixLM-UL-no-L2 & 2.5e-3  & 9.0e-5 & 27.6 \\
        CausalLM-UL-no-L2 & 1.6e-2 & 2.9e-3 & 32.1 \\\hline
    \end{tabular}
    \caption{The test query errors of the unshared-layer (UL) and sharing-layer (SL) transformer-ICLs with or without L2 regularizer on linear regression (LR), non-linear regression (NLR), and multiclass classification (MC) tasks.}
    \label{tab:l2_vs_nol2}
\end{table}

\subsection{The impact of the size of the training data}
\begin{figure*}
  \centering
  \includegraphics[width=0.32\linewidth]{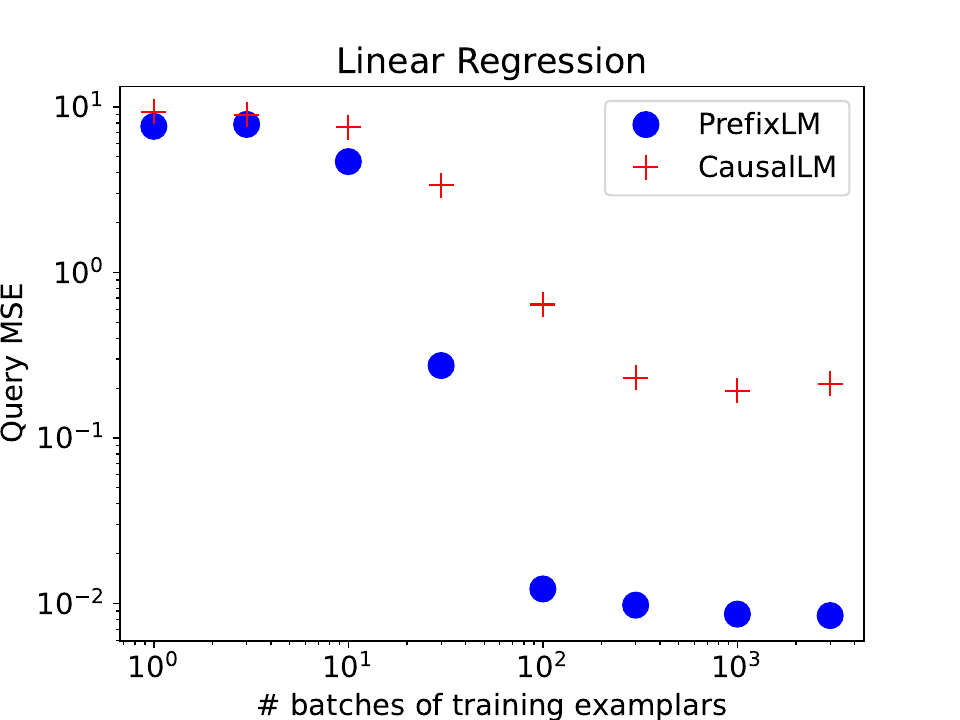}
  \includegraphics[width=0.32\linewidth]{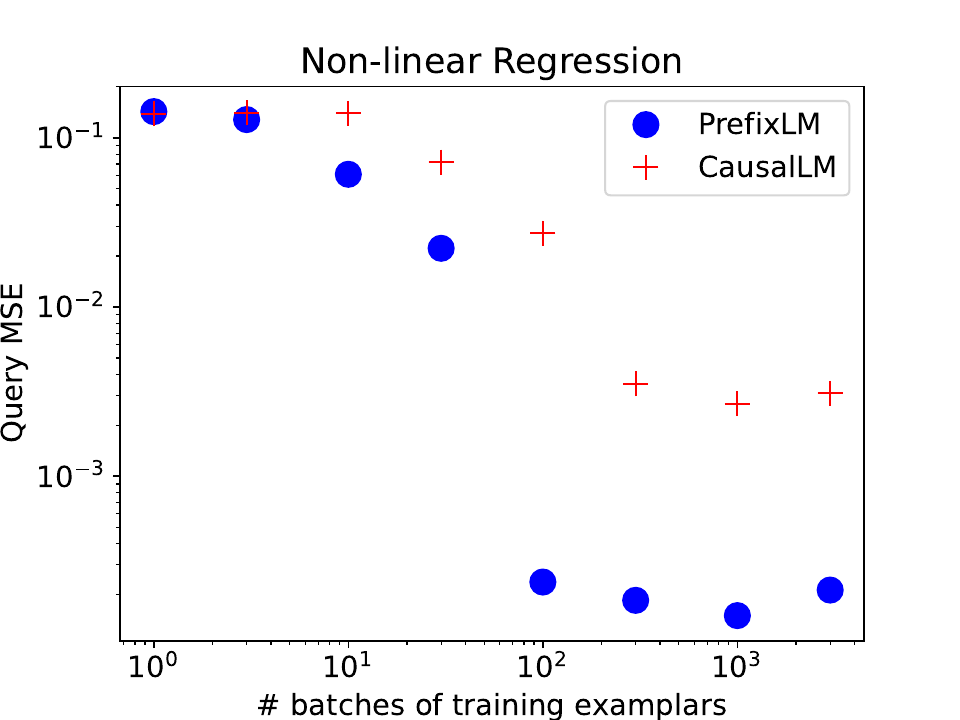}
  \includegraphics[width=0.32\linewidth]{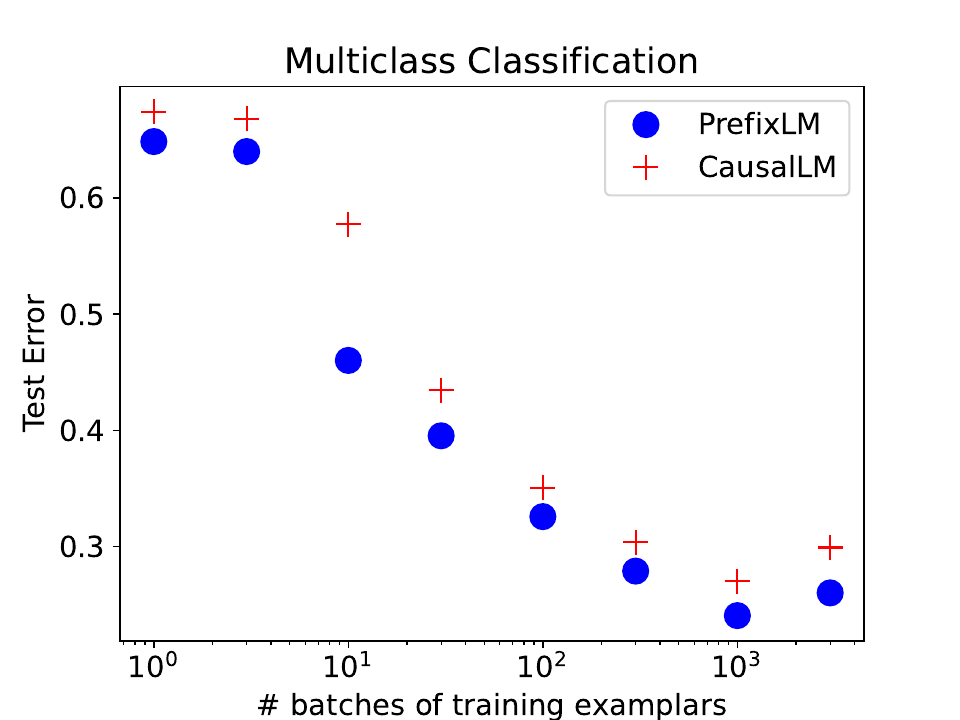}
   \caption{The test query errors of the SL-transformers based prefixLM/causalLM-ICLs with various numbers of training sequences on linear regression (left), non-linear regression (middle), and multiclass classification (right).}
   \label{fig:transformer-batchdata}
\end{figure*}
Here we investigate the performance of prefixLM and causalLM as a function of the number of training samples. In Fig.~\ref{fig:transformer-batchdata}, we plot their after having trained on 10 batches all the way up to 1000 batches (as in Section~\ref{sec:transformer-synthetic}).
We observe that when the amount of training data is low, ICL falls into the memorization regime, in which models perform perfectly on the training data, but do not generalize well to unseen test sequences. 
However, prefixLM transitions to the generalization regime earlier than causalLM, which is reflected by the positions of the largest gap between the two.
(30 batches in LR, 100 batches in N-LR, and 10 batches in MC). 

\subsection{Testing with fewer in-context examples}
In causalLM, different positions in the sequence are trained with different numbers of in-context examples (ICEs). This may bring advantage to pretrained causalLM models when tested with fewer number of in-context examples than it was trained on. 
To compare causalLM and prefixLM in such setting, we use the same models as before that were trained with 40 in-context examples, but test them on fewer (16, 24, 32) in-context examples. 
Note that 16 is the minimum number of examples to solve our 16-dim synthetic regression problems. 
The errors of prefixLM and causalLM are provided in the following Tables~\ref{tab:16ices}, \ref{tab:24ices}, \ref{tab:32ices}, where regression tasks (LR, N-LR) report mean squared errors and the MC task reports the classification error. 
From the tables we see that prefixLM still consistently outperforms causalLM, even when testing with fewer in-context examples than used during training time.
\begin{table}
    \centering
    \begin{tabular}{c|c c c }
        16 Test ICEs & LR  & N-LR & MC \\ \hline
        PrefixLM-SL & 1.01  & 2.1e-2 & 42.8 \\
        CausalLM-SL & 1.76 & 2.7e-2 & 43.3 \\\hline
        PrefixLM-UL & 0.97  & 1.9e-2 & 42.9 \\
        CausalLM-UL & 1.12 & 3.2e-2 & 46.6 \\
    \end{tabular}
    \caption{The test query errors with 16 ICEs on linear regression (LR), non-linear regression (NLR), and multiclass classification (MC) tasks.}
    \label{tab:16ices}
\end{table}
\begin{table}
    \centering
    \begin{tabular}{c|c c c }
        24 Test ICEs & LR  & N-LR & MC \\ \hline
        PrefixLM-SL & 1.4e-1  & 2.0e-3 & 33.4 \\
        CausalLM-SL & 7.0e-1 & 1.0e-2 & 35.9 \\\hline
        PrefixLM-UL & 1.0e-1  & 1.7e-3 & 37.1 \\
        CausalLM-UL & 1.3e-1 & 1.0e-2 & 41.2 \\
    \end{tabular}
    \caption{The test query errors with 24 ICEs on linear regression (LR), non-linear regression (NLR), and multiclass classification (MC) tasks.}
    \label{tab:24ices}
\end{table}
\begin{table}
    \centering
    \begin{tabular}{c|c c c }
        32 Test ICEs & LR  & N-LR & MC \\ \hline
        PrefixLM-SL & 2.4e-2  & 4.7e-4 & 32.4 \\
        CausalLM-SL & 3.1e-1 & 5.0e-3 & 34.6 \\\hline
        PrefixLM-UL & 9.5e-3  & 3.4e-4 & 36.2 \\
        CausalLM-UL & 4.0e-2 & 5.7e-3 & 37.3 \\
    \end{tabular}
    \caption{The test query errors with 32 ICEs on linear regression (LR), non-linear regression (NLR), and multiclass classification (MC) tasks.}
    \label{tab:32ices}
\end{table}

\subsection{Permutation on in-context examples}
We further consider a simple approach for mitigating the problems of causalLM by randomly permuting the in-context examples during training time.
This is motivated by the observation that for causalLM, every permutation representations a different view of the context in the example.
The results of this experiment (Table~\ref{tab:permute_ices}) show that this style of causalLM training indeed improves over the fixed order training setting compared to the unpermuted ICEs (Table \ref{tab:share_vs_noshare}). However, prefixLM still outperforms causalLM in general. 
\begin{table}
    \centering
    \begin{tabular}{c|c c c }
        Permuted ICEs & LR  & N-LR & MC \\ \hline
        PrefixLM-SL & 9.0e-3  & 1.5e-4 & 24.1 \\
        CausalLM-SL & 1.9e-1 & 2.5e-3 & 26.9 \\\hline
        PrefixLM-UL & 2.6e-3  & 9.5e-5 & 26.1 \\
        CausalLM-UL & 1.1e-2 & 1.8e-3 & 26.2 \\
    \end{tabular}
    \caption{The test query errors with randomly permuted ICEs on linear regression (LR), non-linear regression (NLR), and multiclass classification (MC) tasks.}
    \label{tab:permute_ices}
\end{table}

\subsection{In-context learning using PaLM2}
\label{sec:palm2}
Going beyond the publicly available T5 models, we further verify our findings by conducting FLAN-based finetuning experiments using the state-of-the-art PaLM2 family of models~\citep{palm2}. 
PaLM2 models were pretrained with a mixture of objectives that includes different LM types, 
which make them a relatively fair starting point to compare causalLM and prefixLM after finetuning. 
In practice we finetune three sizes of PaLM2 language models: 
Gecko, Otter and Unicorn\footnote{\url{https://blog.google/technology/ai/google-palm-2-ai-large-language-model/}}. 
We use the same default recipe for FLAN-PaLM2 finetuning~\citep{palm2,chung2022scaling} and finetune the PaLM2 checkpoints for either causalLM or prefixLM. We then evaluate the ICL capability of the finetuned models on the Massive Multi-task Language Understanding (5-shot MMLU) tasks~\citep{hendrycks2020measuring}.

Table~\ref{tab:palm} shows that for all three sizes of PaLM2, the MMLU accuracy (average over the 57 tasks) of prefixLM surpasses that of causalLM by more than 3\%. 
This result again empirically verifies that our conjecture generalizes to the practical case, using a state of the art LLM\footnote{
Besides of PaLM2, we also find that any checkpoint that is pretrained with a mixture of prefixLM and causalLM tends to do better with prefixLM for in-context learning.
%Note that PaLM2 was pretrained on a mixture of LM objectives. 
However, we do not claim that prefixLM would necessarily outperform causalLM when using solely causalLM pretrained checkpoints.}.

\begin{table}
    \centering
    \begin{tabular}{c|c c c}
                 & Gecko & Otter & Unicorn\\ \hline
        PrefixLM & {\bf 46.6}  & {\bf 64.8} & {\bf 81.4} \\
        CausalLM & 43.3  & 61.0 & 78.0 \\
    \end{tabular}
    \caption{The average test query accuracies on 5-shot MMLU tasks with FLAN-finetuned PaLM2-Gecko/Otter/Unicorn prefixLM/causalLM checkpoints. \citep{palm2} reported a similar averaged accuracy of 81.2 on Unicorn-PrefixLM.}
    \label{tab:palm}
\end{table}

\subsection{In-context learning with Multimodal Models}
\label{sec:palix}
Lastly, we also demonstrate that prefix attention masks benefit ICL in multimodal models across various settings. We conducted experiments using both 4-shot and 8-shot COCO image captioning tasks on the Karpathy split~\citep{karpathy2015deep} using the PaLI-X model \citep{Chen2023PaLIXOS}, a 55B multimodal pretrained model.

The PaLI-X model employs an encoder-decoder architecture where ViT encoded image tokens and text tokens are fed to the multimodal encoder and decoder to generate outputs. During pretraining, the text prompts were split into two parts. The first part is the input to the multimodal prefix-encoder that self-attends to all the image and text tokens on the encoder side, following the style of prefixLM. The second part is the input to the causal-decoder that self-attends to only the previous text tokens on the decoder side, following the style of causalLM, and cross-attends to encoder tokens.

\begin{table}
    \centering
    \begin{tabular}{l|c c c}
                 &4-shot & 8-shot \\ \hline
        Prefix encoder & {\bf 106.7} & {\bf 107.5} \\
        Block-causal encoder &  104.8 & 106.0 \\ 
        Causal encoder & 102.3 & 104.9 \\   \hline
        Prefix decoder &{\bf 103.9} & {\bf 104.2} \\
        Causal decoder &102.4 & 92.9 \\
    \end{tabular}
    \caption{Cider scores of COCO captioning using various attention masks. The Prefix variant outperforms the Causal ones.
    Note that the official PaLI-X~\citep{Chen2023PaLIXOS} reported a 4-shot Cider of 107.6, which was also based on the prefix encoder mask, but was finetuned with additional image captioning data from the Conceptual Captions 3M dataset~\citep{sharma-etal-2018-conceptual}.}
    \label{tab:causal_prefix_comparison_pali}
\end{table}

The prefix-encoder and causal-decoder nature allows us to consider different variants of the attention masks and placements of the in-context texts to showcase the benefits of prefix attention masks. We design two main categories of few-shot experiments with five self-attention mask settings, detailed below. We finetune the PaLI-X pretrained model using each setting's attention mask with 4-shot Episodic WebLI dataset~\citep{Chen2023PaLIXOS} for 20k steps.

In the first category, we place the few-shot text tokens on the encoder side and study the effect of manipulating the encoder self-attention masks, leaving the causal-decoder unchanged. 
Specifically, considering a 2-shot ICL case for simplicity, we adapt the prefix encoder attention mask $A_{prefix}^{enc}$ in \eqref{eq:prefix_enc} into two causal variants, block-causal and causal encoder attention masks as $A_{b-causal}^{enc}$ in \eqref{eq:bcausal_enc}  and $A_{causal}^{enc}$ in \eqref{eq:causal_enc}. In this case, the block-causal version is more inline with exposing the encoder to the examples one at a time, while the causal one strictly follows auto-regressive attention.
\begin{align} 
A_{prefix}^{enc} = 
\begin{blockarray}{cccccc}
I_1 & T_1 & I_2 & T_2 & I_t \\
\begin{block}{(ccccc)c}
\mathbbm{1}  & \mathbbm{1}  & \mathbbm{1}  & \mathbbm{1}  & \mathbbm{1}  & I_1\\
\mathbbm{1}  & \mathbbm{1}  & \mathbbm{1}  & \mathbbm{1}  & \mathbbm{1}  & T_1\\
\mathbbm{1}  & \mathbbm{1}  & \mathbbm{1}  & \mathbbm{1}  & \mathbbm{1}  & I_2\\
\mathbbm{1}  & \mathbbm{1}  & \mathbbm{1}  & \mathbbm{1}  & \mathbbm{1}  & T_2\\
\mathbbm{1}  & \mathbbm{1}  & \mathbbm{1}  & \mathbbm{1}  & \mathbbm{1}  & I_t\\
\end{block}
\end{blockarray} \label{eq:prefix_enc}
\end{align}
\begin{align} 
A_{b-causal}^{enc} = 
\begin{blockarray}{cccccc}
I_1 & T_1 & I_2 & T_2 & I_t \\
\begin{block}{(ccccc)c}
\mathbbm{1}  & \mathbbm{1}  & \mathbbm{1}  & \mathbbm{1}  & \mathbbm{1}  & I_1\\
\mathbbm{1}  & \mathbbm{1}  & \mathbbm{1}  & \mathbbm{1}  & \mathbbm{1}  & T_1\\
 &  &  \mathbbm{1}  & \mathbbm{1}  & \mathbbm{1}& I_2\\
 &  &  \mathbbm{1}  & \mathbbm{1}  & \mathbbm{1}& T_2\\
 &  &   &   & \mathbbm{1}& I_t\\
\end{block}
\end{blockarray} \label{eq:bcausal_enc}
\end{align}
\begin{align} 
A_{causal}^{enc} = 
\begin{blockarray}{cccccc}
I_1 & T_1 & I_2 & T_2 & I_t \\
\begin{block}{(ccccc)c}
\mathbbm{1}  & \mathbbm{1}  & \mathbbm{1}  & \mathbbm{1}  & \mathbbm{1}  &  I_1\\
  & \diagdown & \mathbbm{1}  & \mathbbm{1}  & \mathbbm{1} & T_1\\
 &  & \mathbbm{1}  & \mathbbm{1}  & \mathbbm{1} & I_2\\
 &  &  & \diagdown & \mathbbm{1}& T_2\\
 &  &  &  & \mathbbm{1} & I_t\\
\end{block}
\end{blockarray} \label{eq:causal_enc}
\end{align}
$I_1$, $I_2$, $I_{t}$ denotes the image tokens for the two shots and the target and $T_1$, $T_2$  denotes the text tokens for the two shots. $\mathbbm{1}$ denotes a matrix of all 1s and ``$\diagdown$'' denote an upper triangular matrix with 1s. A 1 at row $i$ and column $j$ indicates that token $j$ is allowed to attend to token $i$.  
We report the results on few-shot COCO captioning in the top half of Table \ref{tab:causal_prefix_comparison_pali}. We observe consistent improvement over both 4- and 8-shot when changing the encoder attention mask from causal mask, to block causal mask, and then to prefix mask. 

\begin{align} 
A_{causal}^{dec} = 
\begin{blockarray}{cccc}
 T_1 & T_2 & T_t\\
\begin{block}{(ccc)c}
\diagdown & \mathbbm{1}  & \mathbbm{1}  & T_1\\
 & \diagdown & \mathbbm{1}  & T_2\\
 &  & \diagdown & T_t\\
\end{block}
\end{blockarray} \label{eq:causal_dec}
\end{align}

\begin{align} 
A_{prefix}^{dec} = 
\begin{blockarray}{cccc}
 T_1 & T_2 & T_t\\
\begin{block}{(ccc)c}
\mathbbm{1}  & \mathbbm{1}  & \mathbbm{1}  & T_1\\
\mathbbm{1}  & \mathbbm{1}  & \mathbbm{1}  & T_2\\
 &  & \diagdown & T_t\\ 
\end{block}
\end{blockarray} \label{eq:prefix_dec}
\end{align}

Similarly, in the second category, we place the few-shot text on the decoder side and study the effect of manipulating the decoder attention masks, leaving the prefix encoder unchanged. We adapt the causal decoder attention mask $A_{causal}^{dec}$ in \eqref{eq:causal_dec} to prefix attention mask $A_{prefix}^{dec}$ in \eqref{eq:prefix_dec}. Note that all the image tokens from the prefix-encoder side are visible to all text tokens (on the decoder) via cross attention. However, the image tokens cannot attend to the text because of the encoder-decoder architecture. The second half of Table~\ref{tab:causal_prefix_comparison_pali} reports the results of using prefix and causal decoder attention. Even though the decoder is pretrained in the causal manner, with additional finetuning using prefix masks, the new prefix decoder achieves a Cider score of 103.9 in 4-shot ICL, outperforming the finetuned causal decoder by 1.5. Furthermore, the prefix decoder also appears to be more robust when extrapolating to 8-shot evaluation (Cider 104.2), compared to the causal decoder (Cider 92.9).

In summary, the LLM experiments in Section~\ref{sec:flan-t5x} as well as the multimodal experiments in this section show that our conjectures hold up in practice with various types of large-scale models and a wide range of settings.

\end{document}